\documentclass[letterpaper, 10 pt, conference]{ieeeconf}  % Comment this line out
\IEEEoverridecommandlockouts                              % This command is only
\usepackage{graphicx} % for pdf, bitmapped graphics files
\usepackage{epsfig} % for postscript graphics files
\usepackage{mathptmx} % assumes new font selection scheme installed 
\usepackage{times} % assumes new font selection scheme installed
\usepackage{bm}
\usepackage{mathtools}
\usepackage{nccmath}
\usepackage[font=footnotesize]{caption}
\usepackage[font=footnotesize]{subcaption}
\usepackage{float}
\usepackage{amsmath} % assumes amsmath package installed
\usepackage{amssymb}  % assumes amsmath package installed
\usepackage[ruled,vlined,linesnumbered]{algorithm2e}

\DeclareMathOperator*{\argmin}{argmin}
\graphicspath{{./figs/}}
\usepackage{xcolor}
\newtheorem{definition}{Definition}
\newtheorem{remark}{Remark}
\newtheorem{proposition}{Proposition}

\newtheorem{lemma}{Lemma}

\usepackage[normalem]{ulem}

\newcommand{\integernonnegative}{\ensuremath{\mathbb{Z}}_{\ge 0}}
\newcommand{\real}{\ensuremath{\mathbb{R}}}
\newcommand{\realpositive}{\ensuremath{\mathbb{R}}_{>0}}
\newcommand{\realnonnegative}{\ensuremath{\mathbb{R}}_{\ge 0}}

\newcommand{\until}[1]{\{1,\dots, #1\}}
\newcommand{\subscr}[2]{#1_{\textup{#2}}}
\newcommand{\supscr}[2]{#1^{\textup{#2}}}
\newcommand{\setdef}[2]{\{#1 \; | \; #2\}}

\newcommand{\cost}{\subscr{J}{cum}}
\newcommand{\Ar}{\operatorname{Ar}}

\newcommand{\longthmtitle}[1]{\mbox{}{\text{(#1).}}}

\newcommand\oprocendsymbol{\hbox{$\bullet$}}
\newcommand\oprocend{\relax\ifmmode\else\unskip\hfill\fi\oprocendsymbol}

\newtheorem{problem}{Problem}
\renewcommand{\baselinestretch}{0.998}

\title{\LARGE \bf
Planning under non-rational perception of uncertain spatial costs
}

\author{Aamodh Suresh and Sonia Mart{\'\i}nez% <-this % stops a space
  \thanks{A.~Suresh and S.~Mart{\'\i}nez are with the Department of
    Mechanical and Aerospace Engineering, University of California at
    San Diego, La Jolla, CA 92093, USA {\tt\small
      \{aasuresh,soniamd\}@eng.ucsd.edu}}%
  \thanks{We gratefully acknowledge support from  AFOSR grant
  	FA9550-18-1-0158 and DARPA (Lagrange) N660011824027. }  }

\begin{document}

\maketitle
\thispagestyle{empty}
\pagestyle{empty}

\begin{abstract}
	
  This work investigates the design of risk-perception-aware motion-planning
  strategies that incorporate non-rational perception of risks
  associated with uncertain spatial costs. Our proposed method employs
  the Cumulative Prospect Theory (CPT) to generate a perceived risk
  map over a given environment. CPT-like perceived risks and
  path-length metrics are then combined to define a cost function that
  is compliant with the requirements of asymptotic optimality of
  sampling-based motion planners (RRT*). The modeling power of CPT is
  illustrated in theory and in simulation, along with a comparison to
  other risk perception models like Conditional Value at Risk
  (CVaR). Theoretically, we define a notion of expressiveness for a
  risk perception model and show that CPT's is higher than that of
  CVaR and expected risk. We then show that this expressiveness
  translates to our path planning setting, where we observe that a
  planner equipped with CPT together with a simultaneous
  perturbation stochastic approximation (SPSA) method can better
  approximate arbitrary paths in an environment.  Additionally, we
  show in simulation that our planner captures a rich set of
  meaningful paths, representative of different risk perceptions in a
  custom environment. We then compare the performance of our planner
  with T-RRT* (a planner for continuous cost spaces) and Risk-RRT* (a
  risk-aware planner for dynamic human obstacles) through simulations
  in cluttered and dynamic environments respectively, showing the
  advantage of our proposed planner.
	  
\end{abstract}

%%%%%%%%%%%%%%%%%%%%%%%%%%%%%%%%%%%%%%%%%%%%%%%%%%%%%%%%%%%%%%%%%%%%%%%%%%%%%%%%
\section{INTRODUCTION}
\label{sec:introduction}
\paragraph*{Motivation}
\label{sec:motivation}

%\margin{about the ``hi-RRT*''. I suppose we need to explain the
%  acronym somewhere (without saying 'human intuitive'); we could say
%  'human influenced' (?) in the
%  abstract?  }

%  \margina{I changed it to CPT-RRT* to be safe, I don't want the
%    reviewers to pounce on us listening to the word ``human'' like
%    they did last time. Maybe in the future after user studies we can
%    consider renaming to Hi-RRT* :)  -> Good call}
%\margin{Can we really guarantee the optimality of the paths?}

Autonomous robots from industrial manipulators to robotic
swarms~\cite{HC-SH-KML-GK-WB-LEK-ST:05,CW-AMB-AM:18,AS-SM:20}, are
becoming less isolated and increasingly more interactive. 
%in today's
%world.
% These interactions, whose degree of complexity depends on the
% human's involvement in the robot's workspace, directly affect the
% robot path-following behavior. This is the case in the simplest
% scenarios where a decision maker (DM) may participate in the
% guidance of an autonomous system via shared or supervised control,
% such as in robotic surgery, search and rescue operations, or
% autonomous car driving. % , or indirectly as
% is part of the environment.  Path planning is one of the fundamental
% concepts associated with these autonomous robotic systems which is
% needed in order to traverse around their environment to perform
% their tasks.
Arguably, most environments where these robots operate, have an
associated spatial cost, which can lead to a robot's loss or
damage. For example, an oily surface can cause a robot to slip and
collide with a nearby obstacle, resulting in a crash.
In more complex scenarios, a decision maker (DM) may be directly
involved with the motion of an autonomous system, such as in robotic
surgery, search and rescue operations, or autonomous car
driving.  % This includes autonomous cars,
% robotic surgery, search and rescue, or social navigation.
% \color{blue} \margina{modified a bit here.}
The risk perceived from these costs or losses could vary among
different DMs. % In such cases, the decision making of an autonomous
% system w.r.t perceived losses may influence the confidence that a DM
% has on the robot.
This motivates the consideration of richer models that are inclusive
of non-rational perception of spatial costs in motion planning. With
this goal, we aim to study how Cumulative Prospect Theory
(CPT)~\cite{AT-DK:92} can be included into path planning, and compare
its paths with those obtained from
other risk perception models.
% a decision maker (DM) like a human is involved in the
% planning problem either directly through shared control, or
% indirectly as part of the environment. It Both the Psychophysics and
% Behavioral Economics communities have concluded that human DMs
% perceive physical sensation, risk and uncertainty in a fundamentally
% nonlinear manner~\cite{SSS:70,AT-DK:92}, leading to non-rational
% perception models. Empirical paradoxical observations (\color{red}
% cite Allais paradox)\color{black} motivated the latter studies
% resulting in Cumulative Prospect Theory (CPT), a model aimed at
% capturing this non-rational decision making characteristic.

% In the context of path planning, this non-rational decision making can lead
% to alternative plans from existing rational path planners. Motivated
% by this, we aim to study the formulation of motion planning problems
% to account for these qualitative features, while 
% comparing the result with that of rational notions of risk perception. 
%In this work, we study how to adapt CPT to reflect a DM's risk and uncertainty perception into a motion planning approach. Furthermore, these models, which
%depend on parameters, are adapted via a
%learning algorithm to obtain a close CPT model that can represent
%the attitude towards risk and uncertainty of a
%DM. % using the framework of RRT* and call it Hi 
%These methods rely on tuning a risk sensitivity parameter
%\margin{I've changed this paragraph back, I do like this more, but we
% can discuss it}
\color{black}

%-RRT* to
% generate suitable paths in this perceived environment. 
%\margin{complete to just pose the questions, but not state exactly
%what we do}
%\margin{Just saying this should be enough. I've rephrased this not to
% say these CPT models output directly the path of a DM}
\begin{figure}[h]
	\centering
	\begin{subfigure}[t]{0.49\linewidth}
		\includegraphics[width=\textwidth,height=1in]{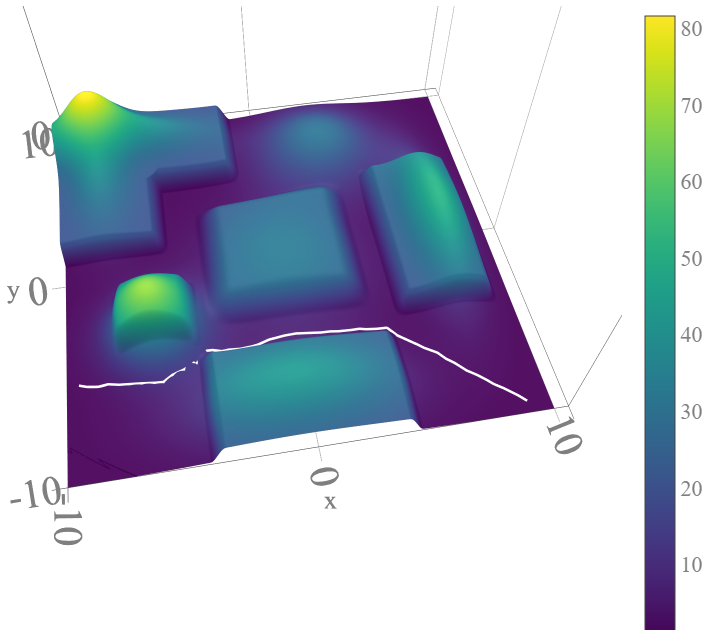}
		\caption{Motion plan with expected risk}
		\label{fig:intro_plan_er}
	\end{subfigure}%
	~
	\begin{subfigure}[t]{0.49\linewidth}
		\includegraphics[width=\textwidth,height=1in]{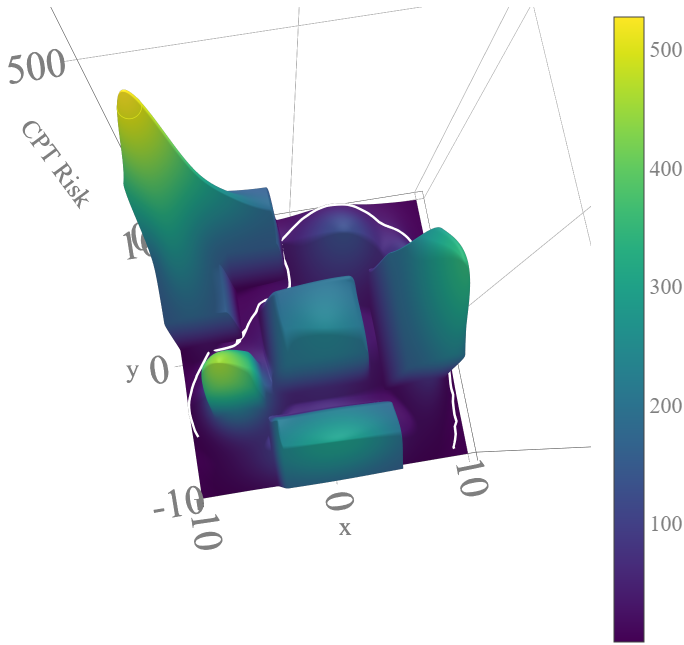}
		\caption{Motion plan with CPT risks}
		\label{fig:intro_plan_cpt}
	\end{subfigure}%
	
	\caption[caption]{Environment perception and sampling-based motion
		planning using a) Rational environment perception using expected risk, b) DM's
		Risk-Averse environment perception with the chosen path in white.}
	\label{fig:intro_plan}
\end{figure}
\paragraph*{Related Work}
\label{sec:related_work}

Traditional risk-aware 
%\margin{is this risk aware planning or
%  uncertainty-aware planning? please clarify except for the 'collision
%  time' others seem to confuse the two concepts}
%  \margina{The community uses "risk-aware" to convey all types of risk including uncertainties.}
path planning considers risk in the form of motion and
state uncertainty~\cite{HK-TB-NMP:12}, collision
time~\cite{JS-SG-TAW:19}, or sensing
uncertainty~\cite{BB-OB:07}. Chance constrained
approaches~\cite{BDL-SK-JPH:13,LB-MO-BCW:12} are used to handle
agent and environment uncertainty in a robust manner, however
discrete polyhedral obstacles are considered which cannot incorporate
continuous spatial costs. Stochastic dynamic programming~\cite{NEDT-JWB:12} in used in dynamic
environments to locally integrate planning and estimation
without optimality guarantees. Moreover, in all the above works, how
the risks and uncertainties are perceived or relatively weighted has
been overlooked. A few recent works~\cite{SS-YC-AM-MP:19}
contemplate %\margin{try to
%reduce the use of the verb 'consider' and the repetition of words}
risk perception models, but assume rational DMs and use coherent risk
measures like Conditional Value at Risk
(CVaR)~\cite{AH-GCK-IY:19}. Unlike CPT's suggestions, these measures
are built using certain axioms that assume rationality and linearity
of the DM's risk
perception~\cite{PA-FD-JME-DH:99}. % In any case, CPT generalizes risk measures such as expected
% risk and CVaR in the sense that we can recover these measures by an
% appropriate choice of CPT setup, catering to a more diverse set of
% DMs.
% \margin{any drawbacks of
%  this approach? is this for rational dm? is modeling limited?}
CPT has been extensively used in engineering applications
like traffic routing~\cite{SG-EF-MBA:10}, network
protection~\cite{ARH-SS:19}, stochastic
optimization~\cite{CJ-PLA-MF-SM-CS:18}, and safe shipping~\cite{LW-QL-TY:18} to model non-rational decision making.  However, CPT is yet to be applied in robotic
planning and control. 

Regarding planning algorithms themselves, RRT*~\cite{SK-EF:11} has
been the basis for many motion planners due to its asymptotic
optimality properties and its ability to solve complex
problems~\cite{BB-TH-SM:18-aa}. Risk~\cite{WC-MQM:17} and
uncertainty~\cite{BE-TS-SDB-AS:16} have been an ingredient of motion
planning problems involving a human, but have been mainly modeled in a
probabilistic manner~\cite{ES-KL-WV-MP:18} with discrete obstacles.
Very few of these works have considered modeling planning environments
via continuous cost maps~\cite{DD-TS-JC:16,TS-BE:15}, while, to the
best of our knowledge, the simultaneous treatment of cost and
uncertainty perception to model a DM's spatial risk profile has
largely been ignored.

\paragraph*{Contributions}
\label{sec:contributions}
%\margin{rewrite the contributions here}
%\margina{done}
%\color{blue}
Our contributions lie in three main areas:
%\begin{enumerate}
%\item \emph{CPT Environment Generator :-} We propose an adaptation of
%  CPT to generate a DM's non-rational perception of an environment which
%  is risky and uncertain.
%\item \emph{CPT Path Planner :-} We generate desirable paths using
%  sampling-based planning on the perceived continuous risky and
%  uncertain environment. We also prove asymptotic optimality of our
%  planner in this continuous risk space using pseudo metric costs,
%  which is different from the traditional RRT* having binary
%  collision checking and a metric cost.
%  %\margin{is this planner essentially different
%    %from RRT*? how is it different?}
%     %\margin{we need to explain the
%    %acronym I guess}
%\item \emph{CPT Parameter Learner :-} Given a DM's path generated by a
%  CPT-model in a known environment, we learn the CPT parameters that
%  characterize this path by implementing SPSA. 
%  %\margin{To be edited carefully}
%\end{enumerate}
%\margin{contributions have to be adapted to reflect theory aspect of
%  CPT planner}
Firstly, we adapt CPT into path planning to model non-rational
perception of spatial cost embedded in an environment. With this, we
can capture a larger variety of risk perception models, extending the
existing
literature.  %\margin{what does the adaptation do? we were more specific
%in the abstract}
Secondly, we generate desirable paths using a sampling-based
(RRT*-based) planning algorithm on the perceived risky
environment. Our planner integrates a continuous risk profile and path
length to calculate path cost, enabling us to plan in the perceived
environment setting above.
%\margin{fill in the
% gaps.} 
Furthermore, the chosen cost satisfies the sufficient conditions for
asymptotic optimality of the planner, leading to reliable and
consistent paths according to a specified risk profile.  We then compare our planner's
performance with T-RRT* (continuous cost space planner) and Risk-RRT*
(risk-aware planner) through simulations in cluttered and dynamic
environments respectively. We show that our proposed planner can
generate better paths in comparison.
%which is different from the traditional RRT*
%having binary collision checking. \margin{this sentence is
%confusing. What is different? the proof or the algorithm?}
Finally, we define the notion of ``expressiveness'' for a risk
perception model and show that CPT's is higher than that of CVaR and
expected risk. Furthermore using SPSA, we show that the expressiveness
hierarchy translates to our path planning setting, where we observe
that a planner equipped with CPT can better approximate arbitrary
paths in an environment. % \color{black} This can be an useful first
% step towards a larger learning scheme for capturing human path
% planning using frameworks like inverse reinforcement learning (which
% is currently out of scope for this work).
%  we don't do).
%\margin{To be edited carefully}
%\end{enumerate}
% Together with these three aspects, an autonomous agent will be able
% to adapt to a specific perception of a given environment using a CPT
% model, and be able to generate more cogent future paths using the
% planner.
\color{black} We clarify that here, we merely examine CPT based
environment perception models for motion planning and leave the
validation of these models with human user studies for future work.
%\margina{Included to the above sentence to clarify. ->
%  OK} % It can also be used to
% estimate a DM's trajectory in a shared environment to plan more
% acceptable paths in scenarios like social navigation.

%\section{Paper Organization}
%{\color{blue}
Fig.~\ref{fig:intro_plan} shows a preview of how a nonlinear DM's
perception of the environment influences the path produced to reach a
goal. While Fig.~\ref{fig:intro_plan_er} shows a rational perception
of the environment using expected risk,
Fig.~\ref{fig:intro_plan_cpt} illustrates a nonlinearly deformed and
scaled surface that reflects the perception of a certain DM using
CPT. 
%In particular, as discussed later, we observe a richer path
%behavior of CPT-based planners as compared with others.
% we can observe that the value of the perceived
% risk is much higher in case of Figure~\ref{fig:intro_plan_cpt} as
% compared to that of Figure~\ref{fig:intro_plan_er}.
%\margin{I would move this paragraph
% to after the contributions}
%\color{blue} \margina{I don't know if something along these lines is
%  needed } Also, we note that most planning algorithms which use
%graph-based or sampling based methods, associate a cost to a path and
%make discrete spatial choices in each iteration. In this work, we use
%CPT based costs instead to induce non-rational decision making when
%faced with risky choices in an RRT* planning framework. This
%non-rational decision making approach can be readily extended to other
%planning frameworks, which is left for future work.  \color{black}
%\margina{maybe not needed}
\color{black}
 
\section{Preliminaries}
\label{sec:prelims}
 Here, we describe some basic notations used in the paper
along with a concise description of Cumulative Prospect Theory. More
details about CPT can be found in~\cite{SD:16}.

\paragraph{Notation}
\label{sec:notation}
We let $\real$ denote the space of real numbers, $\integernonnegative$
the space of positive integers and $\realnonnegative$ the space of non
negative real numbers. Also, $\real^n$ and $\mathcal{C}\subset
\real^n$ denote the $n$-dimensional real vector space and the
configuration space used for planning. We use $ \| .\|$ for the
Euclidean norm and $\circ$ for the composition of two functions $f$
and $g$, that is $f(g(x))=f \circ g(x)$. We model a tree by an
directed graph $G =(V,E)$, where $V=\until{T}$ denotes the set of
sampled points (vertices of the graph), and $E \subset V \times V$,
denotes the set of edges of the graph.
% We use $\mathbb{P}$ to denote the set of $n$ dimensional polygonal
% shapes. We have the environment $\mathbb{S} \subset \mathbb{R}^2$
% which has the obstacle space $\mathbb{O} \subset \mathbb{S}$
% embedded in it. The free space for the swarm to move is then given
% by $\mathbb{F}:= \mathbb{S} \setminus \mathbb{O}$.

\paragraph{Cost and uncertainty weighting using CPT}
\label{sec:PT}

%Cumulative Prospect theory \cite{AT-DK:92} is a Nobel prize winning theory, which
%tries to model human decision making under risk and
%uncertainty. 
%\margin{add citation}
CPT is a non-rational decision making model which incorporates non-linear perception of uncertain costs.
Traditionally it has been used in scenarios of monetary
outcomes such as lotteries \cite{AT-DK:92} and the stock market
\cite{SD:16}.

% In this section, we adapt it for path planning applications.

% Let $x \in \real$ be the wealth/risk variable. Let $x_r \in \real $
% be the reference wealth/risk, then the gain or loss $y \in \real$
% associated with being in state $x$ is given by $y=x-x_r$.
%\margin{I'm eliminating the ref to planning here. We are just
%  summarizing CPT, and will adapt it later to planning.}
%\margina{okay}
Let us suppose a DM is presented with a set of prospects
$\{\rho^1,\dots,\rho^k,\dots,\rho^K \}$, % , with
% $\mathcal{L}^k$ given by the sequence $\mathcal{L}^k = \{(x^k,
% \rho_i^k, p_i^k) \}_{i=1}^{M} $,
representing potential  outcomes
and their probabilities, $\rho^k= \{(
 \rho^k_i, p^k_i) \}_{i=1}^{M} $.  % associated with the prospect location $x^k \in
%\mathcal{C}$. 
 More precisely, there are $M$ possible outcomes associated with a
 prospect $k$, given by $\rho_i^k \in \realnonnegative$, for $i \in
 \until{M}$, which can happen with a probability $p_i^k$. %\margin{I've
 % rephrased this, read to see that makes sense}
%\margin{yes it }
%are associated each with a set of locations $\{x_i^k\}$ in the
%configuration space $\mathcal{C}$, from which the DM chooses the one
%which minimize their expected losses. \margin{do you mean each
%  $\mathcal{L}^k = \{x_k\}$?  I'd say $\mathcal{L}^k = \{(x_i^k,
%  y_i^k, p_i^k) \}_{i=1}^{M_k} $} Each of these prospects
%$\mathcal{L}^k$ also contain outcomes or risk $y_i^k \in
%\realnonnegative$ associated with a location $x_i^k$ in the
%environment and their corresponding probabilities $p_i^k$, where $i
%\in \{1,\ldots,M_k\}$. 
 The outcomes are arranged in a decreasing order denoted by
 $\rho_{M}^k<\rho_{M-1}^k<...<\rho_{1}^k$ with their corresponding
 probabilities, which satisfy $\sum_{i=1}^{M}p_i^k=1$. The outcomes of
 prospect $k$ may be interpreted as the random cost\footnote{CPT has an alternate perception model for random rewards~\cite{SD:16}, which is not used here since we are interested in cost perception. } of choosing
 prospect $k$.
%The
%expected loss for choosing a prospective position $\mathcal{L}^k$ is
%\begin{equation}
%  E(\mathcal{L}^k)=\sum_{i=1}^{M} \rho_i^kp_i^k .
%\end{equation}
%If the DM is completely rational and objective, he/she would choose
%the $k^{th}$ prospect which would minimize the expected risk. If $c$
%represents the human choice then
%\begin{equation}
%  c= \argmin_kE(\mathcal{L}^k). 
%\end{equation}
%\margina{I think we can remove the expected loss from here as we again define expected risk in Section IV, let's keep this section is about the non linear perception of risk and uncertainty using CPT?}
%As pointed out in Section~\ref{sec:introduction}, there is significant
%empirical evidence which suggests that humans in general have a skewed
%notion of not only the risks but also the probabilities associated with the outcomes.

We define a utility function, $v: \realnonnegative \rightarrow \realnonnegative$
%\margin{I prefer you define domain and co-domain of $v$} 
modeling a DM's perceived cost and $w :[0,1]
\rightarrow [0,1]$ as the probability weighting function which
represents the DM's perceived uncertainty.  
% \margin{explain
%  concisely here how CPT uses this to define a risk function
%  associated with outcomes. Explain the logic behind CPT theory, how
%  it captures ``aversion to risk'', ``risk sensitivity'' and
%  ``uncertainty perception'' for a non-rational decision maker. Match
%  the wordings to what comes later, where we distinguish cost from
%  risk. }
 While previous literature have used various forms for these
functions, here we will focus on CPT utility function $v$ taking the
form:
\begin{equation}
  v(\rho)= \lambda\cdot \rho^\gamma,
  \label{eqn:cpt_utilityfn}
\end{equation}
where $0<\gamma<1$ and $\lambda>1$. Tversky and
Kahnemann~\cite{AT-DK:92} suggest the use of $\gamma=0.88$ and
$\lambda=2.25$ to parametrize an average human in the scenario of
monetary lotteries, however this may not hold for our application
scenario. The parameter $\lambda$ represents the coefficient of cost
aversion with greater values implying stronger aversion indicative of
higher perceived costs, as indicated in
Figure~\ref{fig:pre_lambda}. The parameter $\gamma$ represents the
coefficient of cost sensitivity with lower values implying greater
indifference towards cost $\rho$ which is indicated in
Figure~\ref{fig:pre_gamma}.

\begin{figure*}
    \begin{subfigure}[t]{0.24\linewidth}
        \includegraphics[width=\textwidth]{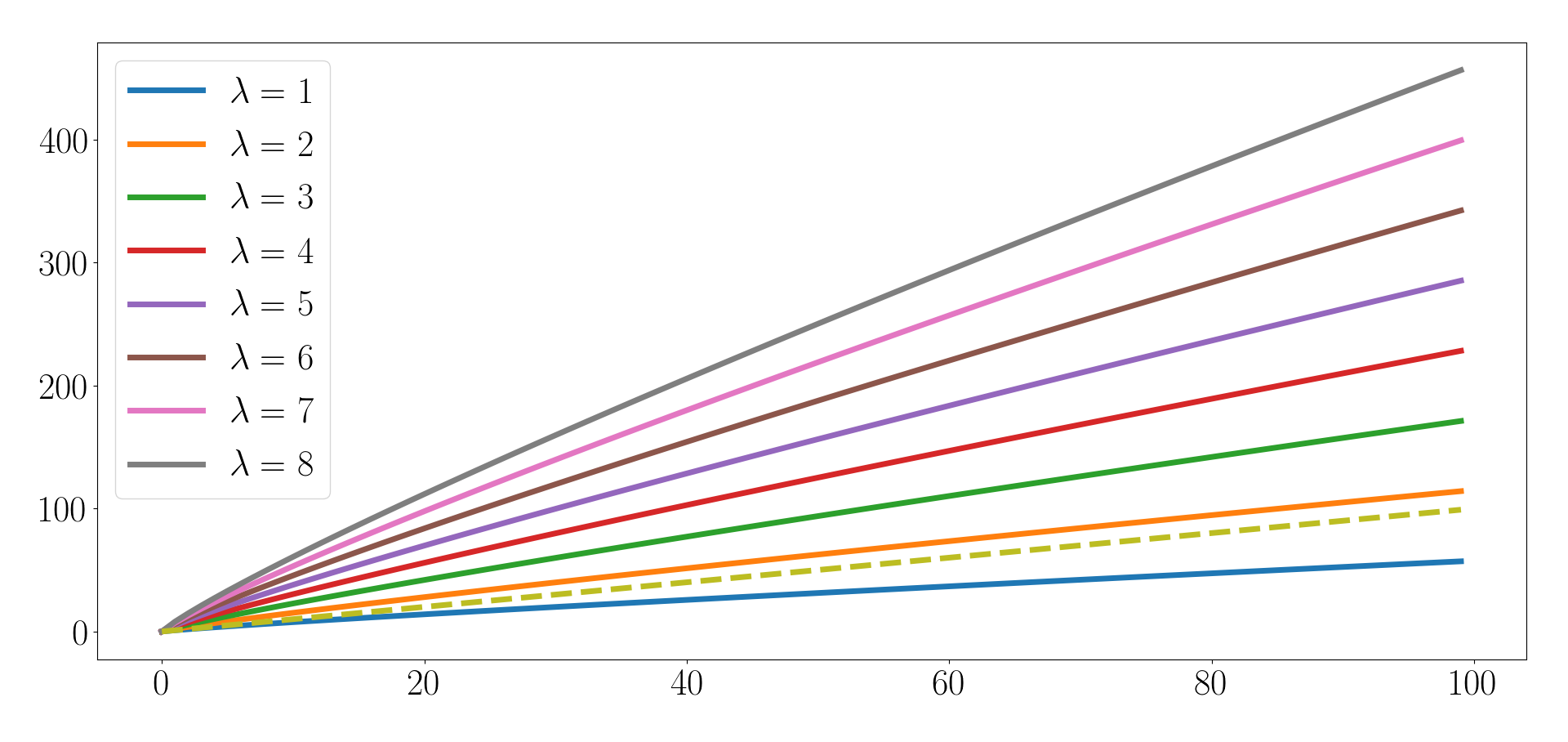}
        \caption{Change in risk aversion $\lambda$ with $\gamma=0.88$ }
        \label{fig:pre_lambda}
    \end{subfigure}%
    ~ 
    \begin{subfigure}[t]{0.24\linewidth}
        \includegraphics[width=\textwidth]{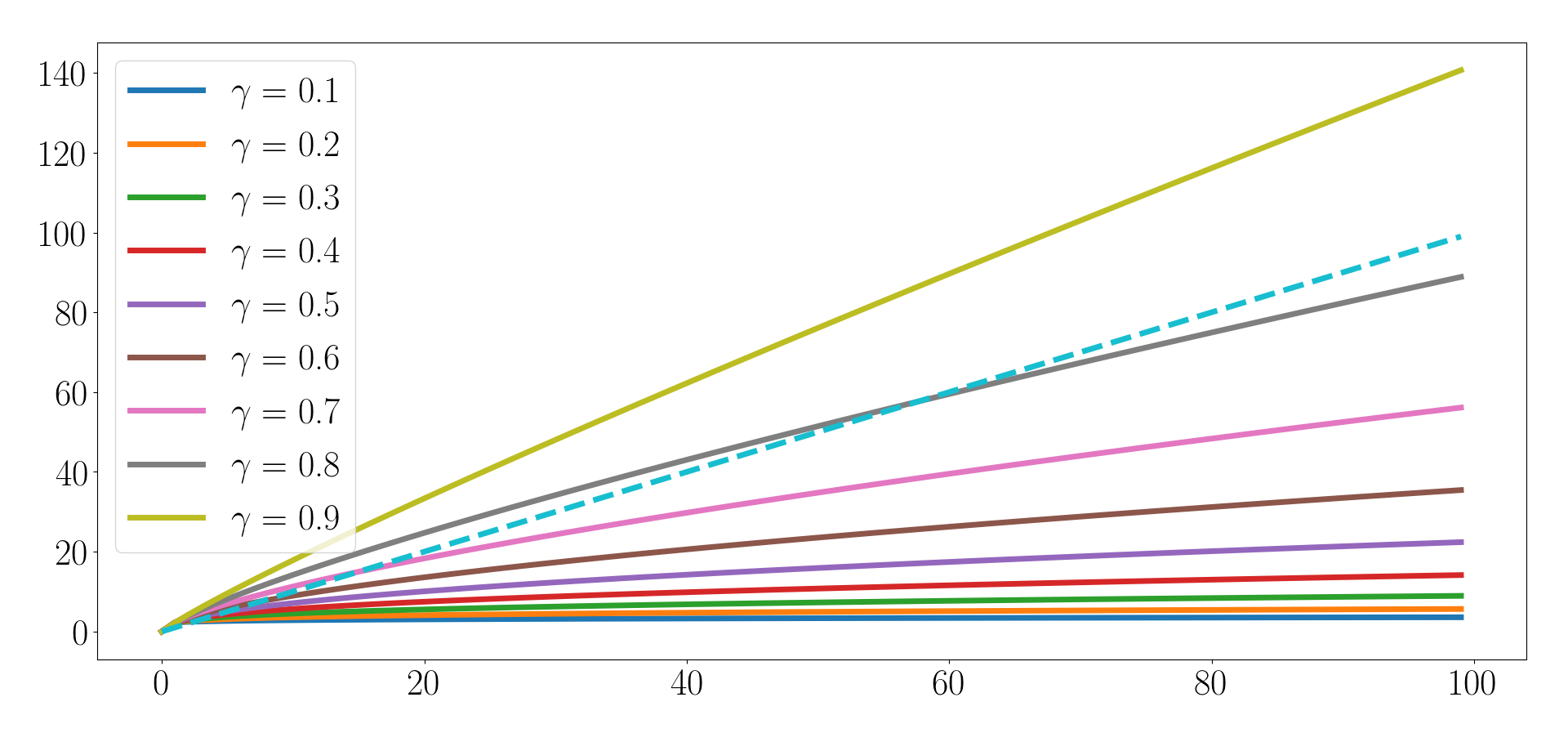}
        \centering
        \caption{Change in risk sensitivity $\gamma$ with $\lambda=2.25$}
        \label{fig:pre_gamma}
    \end{subfigure}
    \begin{subfigure}[t]{0.24\linewidth}
        \includegraphics[width=\textwidth]{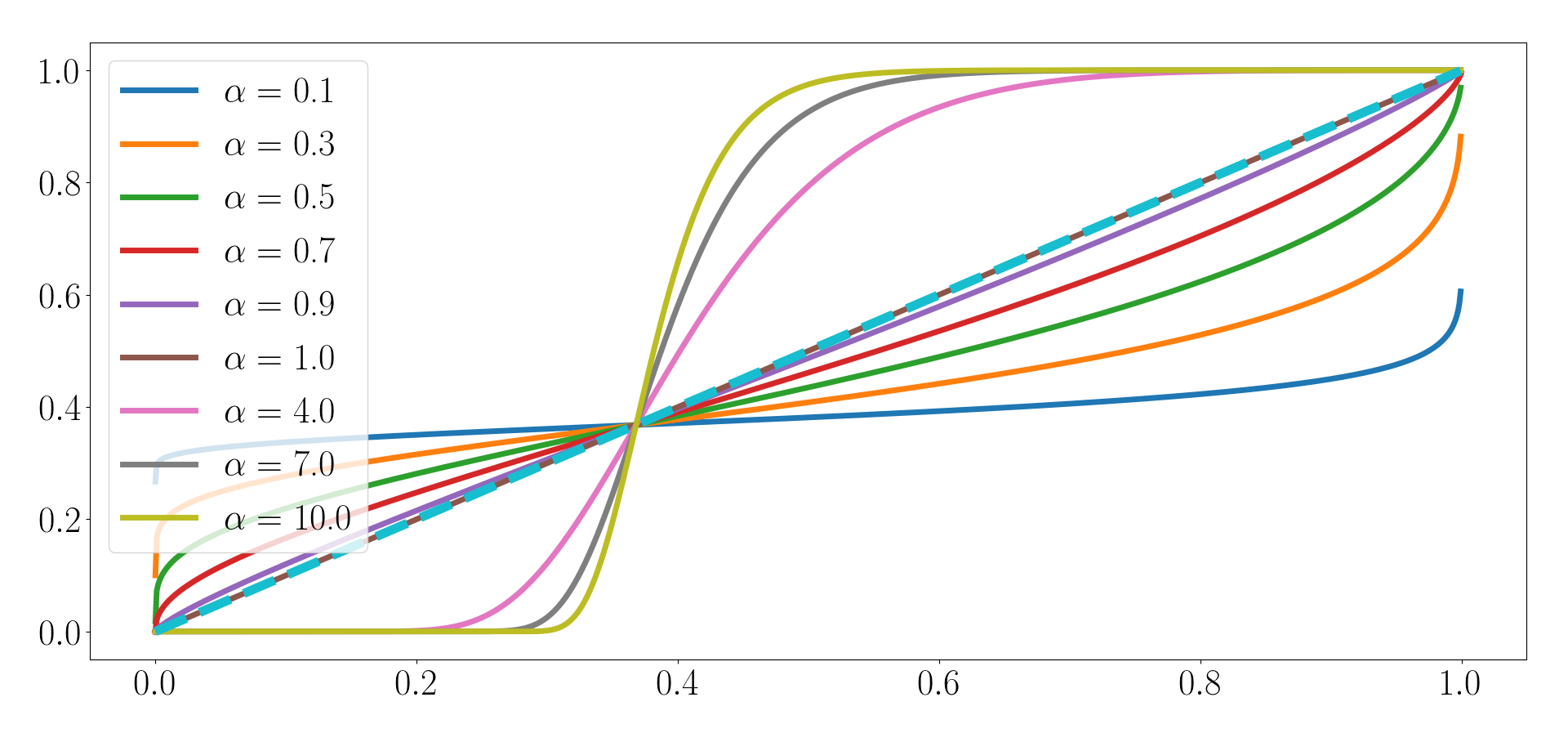}
        \centering
        \caption{Change in $\alpha$ with $\beta=1$}
        \label{fig:pre_alpha}
    \end{subfigure}
    ~
    \begin{subfigure}[t]{0.24\linewidth}
        \includegraphics[width=\textwidth]{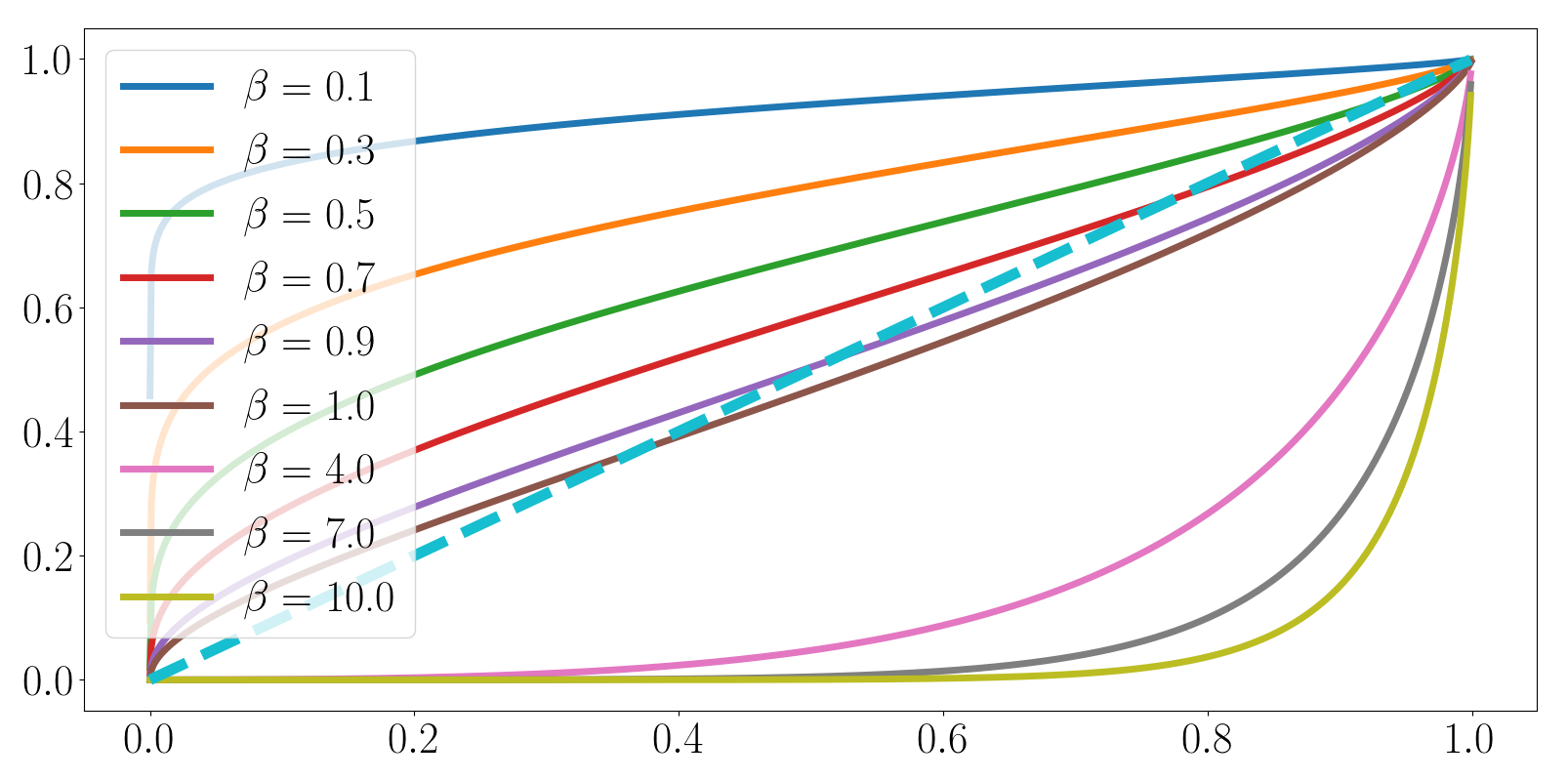}
        \centering
        \caption{Change in $\beta$ with $\alpha=0.74$}
        \label{fig:pre_beta}
    \end{subfigure}
    \caption[Caption]{Variation of risk aversion, risk sensitivity and
      uncertainty perception using CPT. (a)-(b) show risk perception
      with x-axis indicating the associated risk, $\rho$, and the
      y-axis showing the perceived risk, $v$. The dotted line
      indicates the line $v=\rho$. (c)-(d) show uncertainty perception
      with x-axis indicating probabilities $p$ and y-axis showing
      their perception $w$, with the dotted line depicting $w=p$ }
    \label{fig:pre_utility_perception}
\end{figure*}

We will be using the popular Prelec's probability weighting
function~\cite{SD:16,DP:98} indicative of perceived uncertainty, which
takes the form:
\begin{equation}
w(p)=e^{-\beta(-\log p)^{\alpha}},
\alpha>0, \beta>0, w(0)=0.
\label{eqn:cpt_prelec}
\end{equation}
Figures~\ref{fig:pre_alpha}~and~\ref{fig:pre_beta} show changing
uncertainty perception resulting from varying $\alpha$ and $\beta$
respectively. By choosing low $\alpha$ and $\beta$ values, one can get
``uncertainty averse'' behavior with $w(p)>p$, implying that unlikely
outcomes are perceived to be more certain, as seen on
Figures~\ref{fig:pre_alpha}~and~\ref{fig:pre_beta}.  When $w(p)<p$,
``uncertainty insensitive'' behavior is obtained implying that the DM
only considers more certain outcomes, which can be observed with high
$\alpha$ and $\beta$ values.

These concepts illustrate the nonlinear perception of
cost and uncertainty, a DM under consideration can be categorized
by the parameters $\Theta=\{\alpha,\beta,\gamma,\lambda \}$. Using the non-linear parametric perception functions $v$ and $w$, CPT calculates a value function $R^c(\rho)$, indicating the perceived risk value of the prospect $\rho$. This calculation is detailed in Section~\ref{sec:CPT_environment} using our planning setting.

% We will
%be using these to construct a DM's perception of an environment with spatial uncertain costs, which will be detailed in
%Section~\ref{sec:CPT_environment}.

\section{Environment Setup and Problem Statement}

\label{sec:environment_setup}

In this work, we consider spatial sources of risk
embedded in the $\mathcal{C}$ space. Our starting
point is an uncertain cost function $\rho(x)$ that aims to
quantify objectively the (negative) consequences of being at a location $x$ or adopting a certain
decision under uncertainty at a point $x$.

%\margina{I think we should also say that the risk $\rho(x)$ captures
%  the cost of simply being at a location $x$ and add a small example
%  of this like in previous draft.}  \margin{what example would you
%  provide different from the ones below? the fire one? OK, I've edited
%it}
%\margina{yes the fire one, okay looks good.}

% a way, $Y(x)$ is a metric of rational risk.

For example, suppose a robot moves to a location $x$ from $x'$ where
there is an obstacle with certain probability. Then, we can define a
cost measurement as the possible damage to the robot by moving from
$x'$ to $x$ under action $a'$ applied at $x'$. A cost value $\rho(x',x,
a')$ can be defined depending on i) the type of robot (flexible robot
or rigid robot), ii) the probability of having an obstacle in the said
location, and iii) the type of action applied at $x'$ to get to $x$
(e.g.~slow/fast velocity). For simplicity, adopting a worst-case
scenario, we may reduce the previous cost function to a function of
the state $\rho(x) \equiv \max_{x', a'} \rho(x',x,a')$\footnote{Instead of a
  $\max$ operation, one may use an expected operation wrt $x',a'$.}.

%\margin{let's leave this paragraph for the extended version online, we
%  can indicate we have crossed it out. In fact, this whole explanation
%  is a candidate to be removed, but it somehow justifies our bump
%  function later.} 

 As another example, consider a drone
  navigating in a building which is ablaze. In this case, the cost
  function can be proportional to the temperature profile. As sensors
  are noisy, the temperature profile is uncertain, resulting into a
  noisy spatial cost value $\rho(x)$.  Similarly, environmental
conditions that affect the robot's motion may lead to
underperformance. When moving over an icy road, the dynamics of the
robot may behave unpredictably, resulting in a temporary loss of
control and departure from an intended goal state. In this case, the
uncertain cost may be quantified as the state disturbance under a
given action over a given period of time.  For example, for a simple
second-order and fully actuated vehicle dynamics with acceleration
input $a$ which is subject to a locally constant ``ice'' disturbance
$d(x) \approx d$, in a small neighborhood of $x$, we have $(x'-x) = (a
+ d(x))\Delta t^2/2$ for a small time $\Delta t$. Thus, the difference
with an intended state can be measured by the random variable $\rho(x)
= \frac{\|d(x)\| \Delta t^2}{2}$, which encodes information about
$d(x)$ and the unit time of actuation, $\Delta t$. Here $d(x)$ is
uncertain, and can be modeled with prior data or measured with a noisy
sensor as in the temperature profile case.

Prior knowledge in the form of expert inputs and data collected from
sensors can be used to get information about the cost $\rho(x)$,
environmental uncertainty, and the robot's capabilities.
 In this way, icy roads pose much lesser cost in the
  previous sense to a 4WD car with snow tires than a 2WD car with
  summer tires, hence the same cost at a given location could be
  scaled differently depending upon the robot's
  capabilities.
%  \margin{I'd leave the blue paragraph for the online
%  version} 
%\margina{yes this version will be the extended online version}

 In this work, we will assume that the cost at a location $x \in
\mathcal{C}$ has been characterized as a random variable $\rho(x)$
with a mean $\rho_\mu(x) \in \realnonnegative$ and standard deviation
$\rho_\sigma(x) \in \realnonnegative$, for each $x \in
\mathcal{C}$. 
{\color{black} We remark that this cost
	can be constructed from diverse criteria: From nature of
	location (For eg. operating table at hospital being mostly static but highly risky) to dynamic properties (For eg. velocity, state) of close-by
	obstacles.} 
%That is, $\rho$ is a scaled version of a spatial
%Gaussian random variable, which constitutes a large class of spatial
%random variables.
{\color{black} In particular, it is reasonable to approximate
  $\rho_\mu(x)$ via a ``bump function,'' a concept extensively used in
  differential geometry. To fix ideas, consider the
  previous case where a vehicle moves through an ``icy'' environment,
  and assume $\Delta t = 1$. Then, a mean disturbance over a subset $A
  \subseteq \mathcal{C}$ should result approximately into a
  disturbance $\frac{\|d\|}{2} \ell = \supscr{\rho}{max} \ell$, where
  $\ell$ is the portion of the trajectory from $x$ to $x'$ that is
  inside the icy section $A$. As $x$ is farther from $x'$, the
  disturbance reduces its effect on $x$, and the value of $\ell$
  should decrease to zero. In other words, there is a $B$ such that
  $A\subseteq B$, where $B$ is an enlarged region whose boundary
  delimits the uncertain cost area from the certain one (i.e.~outside
  $B$ the cost is zero with low uncertainty). The effect of $\ell$ is
  thus similar to that of a bump function defined with respect to $A$
  and $B$. Bump functions are infinitely smooth, take a positive
  constant value over $A$, which smoothly decreases and becomes zero
  outside $B$. There are many ways of defining bump functions
   on manifolds, such as
  via convolutions, which works in arbitrary dimensions as described
  the following. %\footnote{Bump functions can also be defined in
  % arbitrary Riemannian manifolds.}.. 
  Let $\chi_D:\mathcal{C} \rightarrow \real$ denote the indicator
  function of a subset $D \subseteq \mathcal{C}$, and, given $A$,
  define $f(x) = \frac{\supscr{\rho}{max}}{C} \exp (-\frac{1}{1-
    \|x\|^2}) \chi_A(x)$, with $\int_A \exp(-\frac{1}{1- \|x\|^2}) =
  C$. Then, a bump function based on $A$ and $B$ can be given by the
  convolution $b(x) = \chi_B \star f(x)$, $x \in \mathcal{C}$. This
  function takes a value of $0$ outside $B$, $\supscr{\rho}{max}$
  inside $A$ and a value between $0$ and $\supscr{\rho}{max}$ at the
  points $x \in B \setminus A$.  See Section~\ref{sec:results} for an alternative choice of bump
    function.
  
%   \margin{a justification of
%  bump functions is given here, I hope it works} 

  \color{black} In this work, the notions of ``risk'' and ``risk
  perception'' relate to the way in which the values of $\rho(x)$ are
  scaled and averaged in expectation. That is, risk is a moment of a
  given uncertain function (either $\rho(x)$ or a composition with
  $\rho$). For example, the risk of being at a location $x$ can be
  measured via expected cost; that is $R^e(x)= \mathbb{E}(\rho(x))$,
  which may represent ``expected damage to robot'' with respect to
  uncertainty. However, there are other ways of weighting the
  $\rho(x)$ outcomes to define alternative risk functions, such as
  using CPT.  With this is mind, we proceed to define the following
  three main problems and address them in
    Section~\ref{sec:CPT_environment}, Section~\ref{sec:planning} and
    Section~\ref{sec:learning_SPSA}, respectively.

%\margin{Now we have to modify these problems} 

\color{black}
\begin{problem}\longthmtitle{CPT environment generator}
\label{prob:CPT_environment}
%\margina{if we use $\rho$ for cost then maybe we can use $R(\rho)$ to
 % express perceived risk. }
%\margin{Isn't the CPT risk the Perceived risk? We have a notation of
  %$R^c(\rho)$ for this, isn't that the one you mean?}
%\margina{yes that's what i mean, i'll change it accordingly}
Given the configuration space $\mathcal{C}$ containing the uncertain
cost $\rho$ along with the DM's CPT parameters $\Theta$, obtain a DM's
(non-rational) perceived risk $R^c$ consistent with CPT theory.
\end{problem}

\begin{problem}\longthmtitle{Planning with perceived risk} 
%\margin{we can say ``CPT planner evaluation to capture risk versus
 % alternative metrics''}
\label{prob:path_planner}
Given a start and goal points $x_s$ and $x_g$, compute a desirable
path $P$ from $x_s$ to $x_g$ in accordance with the DM's perceived
risk $R^c$. 
\end{problem}
\begin{problem}\longthmtitle{CPT planner evaluation}
 % \margin{``Evaluation of CPT planner as a model approximator''}
\label{prob:CPT_learner}
Given a configuration space $\mathcal{C}$, and an uncertain cost $\rho$ along with a drawn path
$P_d$, evaluate the CPT planner as a model approximator to generate
the perceived risk $R^c$
%\margin{again,
 % I think you mean $R^c$} 
 representing the path $P_d$.
\end{problem}
% \color{black} Now we will proceed to solve the above problems in
% Section~\ref{sec:CPT_environment}, Section~\ref{sec:planning} and
% Section~\ref{sec:learning_SPSA} respectively in order.
% Problem~\ref{ \margin{refer where
%  is the solution to each problem in the following sections}

\section{Risk perception using~CPT}

%\margina{Change this section according to previous definitions}
%\margin{Make this consistent with the previous section. I've put what
%  I believe is correct in blue}
\label{sec:CPT_environment} Here, we will generate a
DM's perceived risk and address Problem~1. We consider
an uncertain cost $\rho(x)$ is given at every point $x \in
\mathcal{C}$, which we approximate via its first two moments, a mean
value $\rho_\mu(x) \in \realnonnegative$ and a standard deviation
$\rho_\sigma(x) \in \realnonnegative$.  In what follows, we use a
discrete approximation\footnote{The discretization of the random cost
  function is used to be able to use CPT directly with discrete random
  variables. However, it is possible to generalize what follows to the
  continuous random variable case.} of $\rho(x)$ by considering $M \in
\integernonnegative$ bins,
% \margin{why do the bins have to be equally
%  spaced? 
 % I don't think they need to be so}
to obtain a set of possible cost values $\rho(x)\triangleq
\{\rho_1(x),\dots, \rho_M(x) \}$ such that
$\rho_{M}(x)<\rho_{M-1}(x)<...<\rho_{1}(x)$ with their corresponding
probabilities $p(x) \triangleq \{p_1(x),\dots, p_M(x) \}$, such that
$\sum_{i=1}^{M}p_i(x)=1\ \forall x \in \mathcal{C}$. Further, we will
assume that $p_i(x_1)=p_i(x_2)\equiv p_i,\ \ \forall x_1,\ x_2 \in
\mathcal{C}\ , \mbox{and}\ i\in \{1,...,M \}$. In other words, even
though cost values $\rho_i(x)$, $i \in \until{M}$, may change from
point to point in $\mathcal{C}$, the probabilities $p_i(x)$ remain the
same for different $x$. Note that we can do this wlog by discretizing
the continuous RV appropriately, see Algorithm~\ref{alg:CPT_env}. The
function $discretize$ finds $y_i(x)<y_{i+1}(x)$ such that
$\mathbb{P}[y_{i}(x) \le \rho(x) \le y_{i+1}(x)] = p_{i+1}(x) -
p_i(x)$.
%{\color{black} (ADD this info in
%  algorithm: For any random variable, $\rho(x)$, we can choose the
%  $p_i(x)$ a priori, then find $y_i(x)<y_{i+1}(x)$ so that
%  $\mathbb{P}[y_{i}(x) \le \rho(x) \le y_{i+1}(x)] = p_{i+1}(x) -
%  p_i(x)$ and take $\rho_i(x) \triangleq \frac{y_{i+1}(x) -
%    y_i(x)}{2}$.)}\margin{I've added more detail to this part, this
%  can go to the algorithm CPT env generator}
% We note that for any bin size $M$,\eqref{eqn:risk_discretization}, ,
% the width of each bin is the same corresponding probabilities $p(x)$
% will be same at every point $x$, that is $p_i(x_1)=p_i(x_2),\ \
% \forall x_1,\ x_2 \in \mathcal{C}\ \mbox{and}\ i\in \{1,...,M \}$.
% \margin{I would just say we do this
% for simplicity, I'm not sure of the justification we give.}
% \margina{I changed this to be more informative. -> OK}

Now, the expected Risk $R^e(x)$ at a point $x$ is
%\begin{equation}
%R(x)=\frac{1}{M}\sum_{i=1}^M \rho_i(x) p_i(x).
%\label{eqn:exp_risk}
%\end{equation}
%\margina{I don't think $\frac{1}{M}$ is there as the probabilities sum to 1 }

{\small 
\begin{equation}
R^e(x)\triangleq\sum_{i=1}^M \rho_i(x) p_i(x).
\label{eqn:exp_risk}
\end{equation}
} That is, from~\eqref{eqn:exp_risk} we have an expected risk~$R^e :
\mathcal{C} \rightarrow \realnonnegative$
%\margin{I'm
%  changing the notation to $R^e$ for ``expected''. } 
defined over $\mathcal{C}$
%\margin{don't use mapsto
  %for arrows in functions, please just $\rightarrow$}
which is shown in Figure~\ref{fig:intro_plan_er} and corresponds to a
standard or rational notion of risk.
% We see that the expected risk is very similar to the mean risk
% envelope shown in Figure~\ref{fig:intro_mean_risk}. \margin{recall
% how is the mean risk envelope defined. why should we expect a
% difference or a similarity?}  \margina{Yes, true, it is only similar
% here because of small magnitude of uncertainty, I removed the
% sentence. -> OK}

Next, we use the CPT notions developed in Section~\ref{sec:prelims}
%\margin{the
%  label is weird}
to provide a  non-rational perception model of the cost
$\rho(x)$. 
%\margin{everytime we refer to $\rho$ it should be as
%  ``cost'', not ``risk''}% Given a set of $K$ points
% $\{x^1,\ldots,x^K\}$ with $x^k \in \mathcal{C}$, the set of prospects
% can be constructed as $\mathcal{L}^k := \{x^k
% ,\rho_i(x^k),p_i(x^k)\}$, where $i \in \{1,\ldots,M\}$ and $k \in
% \{1,\ldots,K\}$.
% So the DM has to pick a point $x(k)$ which poses the least expected
% risk.
% We have the utility function $v$ from \eqref{eqn:cpt_utilityfn},
% which represents human perceived risk
% %\margin{considerwriting domain/codomain of $v$}
%    and the probability weighing function $w$ from 
% \eqref{eqn:cpt_prelec}, representing human
% perceived uncertainty. 
According to CPT~\cite{AT-DK:92}, 
%\margin{shouldn't we just refer to
%  section II?}
%  \margina{There we just talk about risk and uncertainty perception, here we tell how to calculate CPT value.}
there is a notion of cumulative functions $\Pi :=
\{\pi_{1},...,\pi_{M}\}$ used to non-rationally modify the
  perception of the probabilities $p_i(x)$ in a cumulative
fashion. Defining a partial sum function $S_j(p_1,\dots,p_M)
\triangleq \sum_{i=j}^M p_j$ we have
\begin{equation}
\pi_j = w\circ
S_j(p_1,\dots,p_M) \\ 
- w \circ S_{j+1}(p_1,\dots,p_M), 
\label{eqn:cpt_decision_weights}
\end{equation}
where we employ the weighting function $w$ from~\eqref{eqn:cpt_prelec}.

With this, a DM's CPT risk $R^c:
\mathcal{C} \rightarrow \realnonnegative$ 
%\margin{also changed the
%  notation} \margin{we use the word ``envelope'' in many places. We
%  have risk envelope, which seems to refer to the mean value of
%  $\rho(x)$, and we have uncertainty envelope, which I don't know what
%  it means anymore. Is it the uncertainty about $\rho_\mu(x)$? I'd
%  like to remove that word, or use it very precisely. }
  % \margina{I will remove the word envelope as it is confusing. We will
  %   refer to first moment of cost instead of mean risk envelope,
  %   second moment of cost instead of uncertainty envelope and just
  %   plain risk instead of risk envelope. I hope that makes better
  %   sense. --yes, that makes sense}
%\margin{recall instead of using $\mapsto$, use
 % $\rightarrow$}
 associated to the configuration space is given by:
%\begin{equation}
%  R(x^k)=\frac{1}{M}\sum_{j=1}^M (v \circ \rho_j(x^k))( \pi_j \circ p(x^k)).
%  \label{eqn:cpt_risk}
%\end{equation}

{\small \begin{equation}
  R^c(x)\triangleq\sum_{j=1}^M (v \circ \rho_j(x))( \pi_j \circ p(x)).
  \label{eqn:cpt_risk}
\end{equation}}
%\margin{what is $\nu$? refer to (1)}
%\margina{thats a v}
We note that both functions $R^e$ and $R^c$ are differentiable, which
is important for the good behavior of the planner and which will be used for
the analysis in Section~\ref{sec:planning}.

% \margina{Changed the algorithm a bit to reflect just the calculation
%   of the CPT value instead of calculating CPT value for the entire
%   $\mathcal{C}$}
%\margin{don't forget to mention what is the function 'discretize'}

% \margin{there is no line describing how to go from the continuous RV
%   to the discrete one. Add what I have in parenthesis before line
%   3. Shouldn't $\alpha$, $\gamma$, $\lambda$, and $\beta$ be part of
%   the input list?, what is $\Theta$?}

\begin{algorithm}
\SetAlgoLined
\small 
Input: $\rho_\mu(x),\ \rho_\sigma(x),\ \Theta,\ \{p_1,...,p_M\}$ \\ 
Output : $R^c(x)$ \\ 
% initialization:$M,i$; \\
 %\For{$x \in \mathcal{C}$}{
 %$(\rho(x),p(x)) \leftarrow \mathcal{N}(\rho_\mu(x),\rho_\sigma(x)^2,M)$ \;
 \For{$i \in \{1,...,M\}$}{ 
 $y_i(x),y_{i+1}(x) \leftarrow discretize(p_i(x),p_{i+1}(x))$\;
 $\rho_i(x) \leftarrow \frac{y_{i+1}(x)-y_{i}(x)}{2}$\;
 }
 $w \circ p(x) \leftarrow e^{-\beta(-\log \circ p(x))^{\alpha}}$ \;
 $v \circ \rho(x)\leftarrow \lambda(\rho(x))^\gamma$ \;
 \For{$j \in \{1,...,M\}$}{
 $\pi_j \leftarrow w\circ
S_j(p_1,\dots,p_M)
- w \circ S_{j+1}(p_1,\dots,p_M)$ \;
 }
  $R^c(x) \leftarrow \sum_{j=1}^M (v \circ \rho_j(x))( \pi_j \circ p(x))$ \;
 %$R(x) \leftarrow \sum_{i=1}^M \rho_i(x) p_i(x)$ \;
 
 %}
 %$\mathcal{C}_{CPT} \leftarrow \{\mathcal{C},R(x),R(x)\}$
 \caption{CPT Environment (CPT-Env)}\label{alg:CPT_env}
\end{algorithm}

\color{black}
%\margin{In my version there is a line 13 that is empty why is that?}

% \color{black}
% \margina{does this look okay? ---yes}
%\margin{OK now}
The above concepts are illustrated in
% Figure~\ref{fig:intro_env_perception} and
Figure~\ref{fig:intro_plan}. Given an uncertain spatial cost $\rho$
with the first moment $\rho_\mu$ (Figure~\ref{fig:intro_mean_risk})
and second moment $\rho_\sigma$ (Figure~\ref{fig:intro_variance})
across an environment, the DM's perception can vary from being
rational (i.e.~using expected risk $R^e$ in
Figure~\ref{fig:intro_plan_er}) to non-rational (i.e using CPT risk
$R^c$). By varying $\Theta$, CPT risk $R^c$ can be tuned to represent
risk averse (Figure~\ref{fig:intro_plan_cpt}), risk indifferent
(Figure~\ref{fig:intro_CPT_risk2}) perception, as well as uncertainty
indifferent (Figure~\ref{fig:intro_CPT_unc1}) to uncertainty averse
(Figure~\ref{fig:intro_CPT_unc2}) perception.

This process gives us the 
%new configuration space $\mathcal{C}_{CPT}$
%which is given by $\mathcal{C}_{CPT} =
%\{\mathcal{C},R(x),R(x)\}$.
CPT perceived risk at a point $x$, the process is summarized in
Algorithm~\ref{alg:CPT_env}. It can be seen that
Algorithm~\ref{alg:CPT_env} does not depend on the dimensionality of
the $\mathcal{C}$ space, but on the discretization factor $M$. We will now use the perceived environment for planning in the next section.  \color{black}
%The new space contains the human expected
%risk envelope $R(x)$ and the expected risk envelope $R(x)$ for each
%point $x \in \mathcal{C}$, which is used for planning in the following
%section.

\section{Sampling-based Planning using perceived risk}
\label{sec:planning}
Here, we will use CPT notions to derive new cost functions, which will
be used for planning in the DM's perceived environment generated in
Section~\ref{sec:CPT_environment}.  In traditional RRT* optimal
planning is achieved using path length as the metric. In our setting,
the notion of path length is insufficient as it does not capture the
risk in $\mathcal{C}$. Thus, we define cost functions that a) take
into account risk and path length of a path, and b) satisfy the
requirements that guarantee the asymptotic performance of an
RRT*-based planner.

% We are interested in
% finding the cost associated with a path $P$ which takes into account
% risk, uncertainty as well as the path length.
\paragraph{Path costs functions}
Let two points $x,y \in \mathcal{C}$ be arbitrarily close. A decrease
in risk is a desirable trait, hence it is reasonable to add an
additional term in the cost only if $R(y)-R(x)\ge 0$, which indicates
an increase in DM's perceived risk by traveling from $x$ to $y$.
Consider the set of parameterized paths $\mathcal{P}(\mathcal{C})
\triangleq \setdef{\eta:[0,1] \rightarrow \mathcal{C}}{\eta(0) = x, \,
  \eta(1) = y}$. First, we first define the cost
$J^c:\mathcal{P}(\mathcal{C}) \rightarrow \realnonnegative $
%\margin{let's use big J for cost, then we have $J^c$ and $J^e$}
%$c: \mathcal{C}\times\mathcal{C}
%\rightarrow \realnonnegative$
%\margin{1) use domain and codomain of $c$, and
 % (2), we don't use the notation $\real^+$ in this paper. Be
  %consistent with notations.}
of a path $\eta \in \mathcal{P}(\mathcal{C})$. Consider a
discretization of $[0,1]$ given by $\{0, t_1,t_2,\dots, t_L=1\}$ with
$t_{ \ell+1} - t_\ell = \Delta t$, for all $\ell$. Then,
%\begin{equation} 
%  c(x,y)= \max \{0,R(y)-R(x)\} + \delta L(x,y),
%\label{eqn:cpt_costs}
%\end{equation} 
a discrete approximation of the cost over $\eta$ should be:
%\begin{align*} 
%  c(\eta) & \approx \sum_{\ell=1}^L \max \{0,R(t_{\ell+1}) -
%  R(t_\ell)\} +
%  \delta L(\eta)\\
%  & = \Delta t \sum_{\ell=1}^L \max
%  \{0,\frac{R(\eta(t_{\ell+1}))-R(\eta(t_\ell)}{\Delta t}\} +
%  \delta L(\eta),
%%\label{eqn:cpt_costs}
%\end{align*}
%\margina{shortening equations}
%\margin{OK for these. However, have in mind that removing equations
%  removes precision and clarity}

{\small \begin{equation*}
J^c(\eta) \approx \Delta t \sum_{\ell=1}^L \max
  \{0,\frac{R^c(\eta(t_{\ell+1}))-R^c(\eta(t_\ell))}{\Delta t}\} +
  \delta L(\eta),
%\label{eqn:cpt_costs}
\end{equation*}}
where $L(\eta)$ denotes the arc-length of the curve $\eta$, and
$\delta \in \realnonnegative$ is a constant 
encoding an urgency versus risk tradeoff. 
%\margin{not sure if the word
%``riskiness'' exists. I think the word is just risk. It sounds like
%'awesome' and 'awesomeness' to me} % This can be tuned to
% adjust to a specific environment or robot's capabilities. 
The greater
the $\delta$ value, the greater is the urgency and hence path length
is more heavily weighted whereas, smaller $\delta$ indicates greater
prominence towards risk. %  For example, a mars rover in a harsh
% environment would prefer a path with minimum risk, giving less
% importance to urgency and hence will have a smaller $\delta$
% value. Whereas, an autonomous car whose passenger is late for a
% meeting will be willing to take a few risks and try to reach the
% destination urgently, thus preferring a high $\delta$ value.
The choice of $\delta$ will be discussed in
Section~\ref{sec:results}.
%\margin{add discussion on $\delta$
%  there. I've cut the description down.}  \color{black}
\color{black}
   By taking limits in the previous expression, and due to the continuity
and integrability of $\max$, we can express $J^c(\eta)$ as:

{\small
\begin{align} 
 J^c(\eta)= & \lim_{\Delta t \rightarrow 0}  \Delta t \sum_{\ell=1}^L
 \max \{0,\frac{R^c(\eta(t_{\ell+1}))-R^c(\eta(t_\ell))}{\Delta
   t}\} +
 \delta L(\eta) \nonumber \\
 & = \int_0^1 \max \{0,\frac{d}{dt} (R^c(\eta(t))\} \text{d} t  + \delta
 L(\eta) = \nonumber \\
& \int_0^1 \max \{0, (R^c)'(\eta(t))\cdot \eta'(t) \}\text{d} t + \delta
L(\eta).
\label{eqn:cpt_costs}
\end{align}
}
%\margina{cut down the equations to make space} \margin{I've put these
%  back. Let's remove them as last resort.}

From here, the cost of traveling from $x$ to $y$ is given by

\begin{align*}
  J^c(x,y) \triangleq \min_{\eta \in \mathcal{P}(\mathcal{C}):
    \eta(0) = x , \eta(1) = y}
  J^c(\eta).
\end{align*}

Similarly, the path cost using expected risk
$J^e:\mathcal{P}(\mathcal{C}) \rightarrow \realnonnegative$ can be
obtained by replacing the CPT cost $R^c$ in~\eqref{eqn:cpt_costs} with
the expected risk $R^e$ as calculated in~\eqref{eqn:exp_risk}. 
\begin{remark}\longthmtitle{Monotonicity}
  It can be verified that the costs $J^c$ and $J^e$ satisfy monotonic
  properties in the sense that 1) they assign a positive cost to any
  path in $\mathcal{P}(\mathcal{C})$, and 2) given two paths $\eta_1$
  and $\eta_2$, and their concatenation $\eta_2|\eta_1$, in the space
  $\mathcal{P}(\mathcal{C})$, it holds that $J^c(\eta_1) \le
  J^c(\eta_1 |\eta_2)$ (resp.~$J^e(\eta_1) \le J^e(\eta_1 |\eta_2)$),
  (due to the additive property of the integrals) and 3) $J^c$
  (resp.~$J^e$) are bounded over a bounded $\mathcal{C}$.
  %\margin{now this sounds good to me. }
\label{rem:monotonicity}
\end{remark}

%\begin{remark}\longthmtitle{Boundedness}
%  The costs \eqref{eqn:cpt_costs} and \eqref{eqn:exp_costs} satisfy monotonic properties in the sense: for any $x,y,z \in \mathcal{C}$ we have $c(x,y)\leq c(x,y)+c(y,z)$ and similarly $c(x,y)\leq c(x,y)+c(y,z)$
%\end{remark}
% \margin{complete as before to refer to $R$ and $\delta$}

%\color{black}
%Now we can also introduce uncertainty in the cost
%$c(x,y)$. Logically, regions with higher uncertainty need to be
%assigned with higher costs. Entropy is a standard notion which
%captures this uncertainty. The higher the uncertainty, the higher is the
%entropy and this is the highest when the distribution is uniform. Now, the
%change in uncertainty between two points $x,y$ can be captured by
%relative entropy or the KL-Divergence ($\KL(x,y)$) 
%\margin{define in prelim} calculated between
%the uncertainty distributions along these two points. The
%KL-Divergence term can be added to the cost function to get:
%
% \begin{equation}
%c(x,y)= \max \{0,R(x)-R(y)\} + \delta_1 L(x,y) + \delta_2 KL(x,y).
%\end{equation}
\color{black}

%We will use this notion of {\color{black} human intuitive} cost in a
%sampling based path planning. We will introduce the proposed algorithm
%to perform human intuitive planning, which we call Hi-RRT*. In the
%case of the traditional RRT*, these costs will come into play when
%choosing candidate parent and children. . 
%\margin{this whole paragraph can be
%  omitted or rephrased and moved to the next subsection. And you are
%  saying everywhere 'human intuitive'. Please rephrase because we
%  can't give a wrong impression.}
%  \margina{okay makes sense, i removed it -> OK}
%\margin{you didn't use command lines inside the algorithm (e.g. for
  %cost) or subscr. I've replaced a few, but go ahead and replace the
  %others and change to the new notation}

\begin{algorithm}
\SetAlgoLined
Input: $T,x_s,x_g\ $; 
Output : $G(V,E), P$ \\
$V\leftarrow x_s$, $E \leftarrow \phi $, $\cost^c(x_s) \leftarrow 0 $\;
 \For{$i \in \until{T}$}{
  $G \leftarrow (V,E)$;
  $\subscr{x}{rand} \leftarrow Sample()$\;
 $\subscr{x}{nearest} \leftarrow Nearest(G,\subscr{x}{rand})$;
 $\subscr{x}{new} \leftarrow Steer(\subscr{x}{nearest},\subscr{x}{rand}) $\;
 $V \leftarrow V \cup {\subscr{x}{new}} $;
 $\subscr{x}{min} \leftarrow \subscr{x}{nearest}$\;
 $\subscr{X}{near} \leftarrow Near(G,\subscr{x}{new},\subscr{\gamma}{RRT*},d)$\;
 $\subscr{c}{min} \leftarrow \cost^c(\subscr{x}{nearest}) +\supscr{J}{c}(\subscr{x}{nearest},\subscr{x}{new})$ \;
 \For{$\subscr{x}{near} \in \subscr{X}{near}$}{
 $c' \leftarrow \cost^c(\subscr{x}{near})+J^c(\subscr{x}{near},\subscr{x}{new})$\;
 \If{$c'<c_{min}$}{
 $\subscr{x}{min} \leftarrow \subscr{x}{near}$;
 $\subscr{c}{min} \leftarrow c'$ \;
 }
 }
 $\cost^c(\subscr{x}{new}) \leftarrow \subscr{c}{min}$ ; 
 $E \leftarrow E \bigcup (\{\subscr{x}{near},\subscr{x}{new}\})$\;
 \For{$\subscr{x}{near} \in \subscr{X}{near}$}{
 $c' \leftarrow \cost^c(\subscr{x}{new})+J^c(\subscr{x}{new},\subscr{x}{near})$\;
 \If{$c'<\cost^c(\subscr{x}{near})$}{
 $\subscr{x}{par} \leftarrow Parent(\subscr{x}{near},G)$\;
 $E \leftarrow (E \setminus (\{\subscr{x}{par},\subscr{x}{near}\})) \bigcup (\{\subscr{x}{new},\subscr{x}{near}\})$ \;
 $\subscr{X}{chld} \leftarrow Children(\subscr{x}{near},G)$\;
 \For{$\subscr{x}{chld} \in \subscr{X}{chld}$}{
 $\cost^c(\subscr{x}{chld}) \leftarrow~\cost^c(\subscr{x}{chld}) - \cost^c(\subscr{x}{near}) + c' $
 }
 $\cost^c(\subscr{x}{near}) \leftarrow c'$
 }
 }
 }
$P \leftarrow Path(G,x_s,x_g)$ \;
\caption{CPT-RRT*}\label{alg:hi-rrt*}
\end{algorithm}

\paragraph{Proposed Algorithm}

Now we have all the elements to adapt RRT* to our problem
setting. Given $\mathcal{C}$, a number of iterations $T$ and a
start point $x_s \in \mathcal{C}$, we wish to produce graph $G(V,E)$,
which represents a tree rooted at $x_s$ whose nodes $V$ are sample
points in the configuration space and the edges $E$ represent the path
between the nodes in $V$. Let $\cost^c: \mathcal{C} \rightarrow
\realnonnegative$
%\margin{I've created an operator name for cost,
 % otherwise the text used in math mode doesn't look right}
be a function that maps $x \in \mathcal{C}$ to the cumulative cost to
reach a point $x$ from the root $x_s$ of the tree $G(V,E)$ using the
CPT cost metric~\eqref{eqn:cpt_costs}. Similarly we define $\cost^e:
\mathcal{C} \rightarrow \realnonnegative$ for the expected cost
function $J^e$.

\begin{remark}\longthmtitle{Additivity}
  The cumulative costs $\cost^h$ and $\cost$ 
  %\margin{I've removed the
    %arguments $(.)$, we don't need to add it}
   are additive with
  respect to costs $J^c$ and $J^e$ in
  the sense that: for any $x \in V$ we have
  $\cost^c(x)=\cost^c(Parent(x))+J^c(Parent(x),x)$ and similarly
  $\cost^c(x)=\cost^e(Parent(x))+J^c(Parent(x),x)$.
  \label{rem:additivity} 
\end{remark}
The other basic functional components of our algorithm CPT-RRT*
(Algorithm~\ref{alg:hi-rrt*}) are similar to RRT*,
  and we briefly outline it out here for the sake of completeness:
%\margin{summarize this part as much as possible}
%\margin{run a spell check, algorithm is
  %not well spelled}
%\margin{refer to its name and to the table} 
%are briefly explained as follows:
\begin{itemize}
\item $Sample()$: Returns a pseudo-random sample $x \in \mathcal{C}$
  drawn from a uniform distribution across $\mathcal{C}$.
   Other risk-averse sampling schemes as in
  \cite{DD-TS-JC:16} may be employed. However, such schemes lead to 
  conservative plans, which may not be suitable for all risk profiles.
  \color{black}
\item $Nearest(G,x)$: Returns the nearest node according to the
  Euclidean distance metric from $x$ in tree~$G$.
  %\margin{nearest according to euclidean metric?}
\item $Steer(x_1,x_2)$ returns \[ \begin{cases}
      x_2, & \text{if $\|x_2 - x_1\| \leq d$} \\
      x_1 + d \frac{x_2-x_1}{\|x_2 - x_1\|}, & \text{otherwise.}
    \end{cases}
\]
\item $Near(G,x,\gamma_{RRT^*},d)$: returns a set of nodes $X \in V$
  around $x$, which are within a radius as given in~\cite{SK-EF:11}.
\item $Parent(x,G)$: Returns the parent node of $x$ in the tree~$G$.
\item $Children(x,G)$: Returns the list of children of $x$ in~$G$.
\item $Path(G,x_s,x_g)$: Returns the path from the nearest node to $x_g$
  in $G$ to $x_s$. 
%  It first uses the $Nearest(G,x_g)$ function, and
%  then applies $Parent(x,G)$ function recursively to obtain path from
%  $x_s$ to the node nearest to $x_g$.   
\end{itemize}
%Since we redefined the costs, we have to explain better how we
%  approximate the costs over paths What we can say is that for each
We note that in order to compute $J^c$ for each path, we approximate
the cost as the sum of costs over its edges, $(x_1,x_2)$, and for each
edge we compute the cost as the differences $\max \{ 0, R^c(x_2) -
R^c(x_1)\} + \delta L(x_1, x_2)$, where the latter is just the length
of the edge. Then, this approximation will approach the computation of
the real cost in the limit as the number of samples goes to
infinity. The values $R^c$ are evaluated according to
Algorithm~\ref{alg:CPT_env}. Our proposed CPT-RRT* algorithm augments RRT* algorithm in
the following aspects: we consider a general continuous cost profile
which leads to no obstacle collision checking.  We also consider both
path length and CPT costs for choosing parents and rewiring with the
parameter $\delta$ which serves as relative weighting between CPT
costs and Euclidean path length.

%\margin{How does this algorithm differ from the standard RRT*? This is
%  something a reviewer will ask for sure. We don't have a collision
%  check primitive. Could it be added? And is there a way of
%  improving the RRT* we use for our case? For example, could we relate
%  a collision check with a certain low risk perception to limit the
%  collision checks and provide some guarantees or not surpassing a
%  given risk? Maybe you can think of other options here. I feel the
%  departure from the standard RRT* is not significant enough so that a
%  motion planning person may not accept this as a  contribution.  }
%  \margina{Okay, I have added the above explaining the differences. 
%  }
% After the creation of the tree $G$ rooted at $x_s$, we wish to find
% the path $P$ from $x_s$ to our goal $x_g$. To do this, we first find
% the nearest node to $x_g$ in $G$ using the $Nearest(G,x)$
% function. Now using the $Parent(x,G)$ function recursively we obtain
% the path from $x_s$ to a node nearest to $x_g$.

\begin{remark}\longthmtitle{ER-RRT*}
  We can obtain the expected risk version of
  Algorithm~\ref{alg:hi-rrt*} by replacing cost function $J^c$ by
  $J^e$ and following the same procedure as
  Algorithm~\ref{alg:hi-rrt*}.
  %The results are visualized in
  %Figure~\ref{fig:intro_plan_er}.
\label{rem:er_rrt*}  
\end{remark}
%\margin{I'd like to state more formally this remark, as a lemma or a
 % theorem, but the truth is that the statement in Karaman's paper is
  %not very formal either. I'm going to try}
\begin{lemma}\longthmtitle{Asymptotic Optimality} 
  Assuming compactness of $\mathcal{C}$ and the choice of $\gamma_{RRT^*}$ according to Theorem~38 in~\cite{SK-EF:11} %\margin{here choose the paramter as in RRT* as in
    %the statement of theorem 39 in Karaman's paper}. 
    , the CPT-RRT*
  algorithm is asymptotically optimal.
\end{lemma}
\begin{proof}
  It follows from the application of Theorem~38 in~\cite{SK-EF:11},
  and the conditions required for the result to hold. More precisely,
  the cost functions are monotonic (which follows from
  Remark~\ref{rem:monotonicity}), it holds that $c(\eta) = 0$ iff
  $\eta$ reduces to a single point (resp.~the same for $c$), and the
  cost of any path is bounded. The latter follows from the compactness
  of $\mathcal{C}$ and continuity of the cost functions. In addition,
  the costs are also cumulative, due to the additivity
  property in Remark~\ref{rem:additivity}. Finally, the result also requires
  the condition of the zero measure of the set of points of an optimal
  trajectory. This holds because both costs include a term for path
  length.
\end{proof}

%\margin{change this paragraph}
Simulation results of CPT-RRT* algorithm are presented
  in Section~\ref{sec:results_env_plan}. Next we describe our proposed
  method to evaluate and compare risk perception models in our
 setting.

\section{CPT-planner parameter adaptation}
\label{sec:learning}
\label{sec:learning_SPSA}
% In the previous sections we have explained how to plan a risk
% sensitive path using our algorithm CPT-RRT*, detailed in
% Algorithm~\ref{alg:hi-rrt*}.
 %\margin{so here you are admitting
  %we directly use RRT*. Why not refer to your algorithm's name's?}
\color{black} 
%\margin{I'm changing this to make it consistent with our
%  new plan of ``learning''. Instead of ``learn'' we are going to
%  use the word ``adapt''.}

In this section,  we theoretically compare CPT value function with other risk perception functions and \color{black} we describe an algorithm that can adapt the CPT
parameters of the planner to approximate arbitrary paths in the
environment. By doing so, we aim to evaluate the expressive power of
the CPT risk perception model, both theoretically and in a motion planner \color{black} by comparing its capability to approximate single and arbitrary paths in the environment versus other approaches using different risk perception models.
% \margin{maybe we need to run the spsa for
%  different example paths then.}
%  \margina{Yes, that is what i am doing}

If successful, this method could be used as a first ingredient in a
larger scheme aimed at learning the risk function of a human decision
maker\footnote{Just for offline planning, or in situations where the
  human does not update the environment online as new information is
  found.} using techniques such as inverse reinforcement learning
(IRL). We recall that IRL requires either discrete state and action
spaces or, if carried out over infinite-dimensional state and action
spaces, a class of parameterized functions that can be used to
approximate system outputs. Since our planning problem is defined over a
continuous state and action space, the class of CPT planners for a
parameter set could play the role of a function approximation class
required to apply IRL. Then, as is done in IRL, a larger collection of
path examples can used to learn the best weighted combination of
specific CPT planners in the class. While certainly of interest, this
IRL question is out of the scope of this work, and we just focus on
analyzing the expressive power of the proposed class of CPT
planners. Having a good expressive power is a necessary prerequisite
for the class of CPT planners to constitute a viable function
approximation class. Firstly, we will define the notion of expressiveness and compare the expressiveness risk perception models from a theoretical point of view. Next, we will describe an approach to compare expressiveness in a path planning setting using SPSA.

\subsection{Expressiveness for a risk perception model.}
Let $\rho$ be a random cost variable with an associated probability distribution. Let $R$ be a risk value function (with $R(\rho) \in \realnonnegative $) which associates a real value to the random cost variable $\rho$. We can compare the expressiveness of two risk perception models by comparing the range space of their respective risk value functions. 
\begin{definition}\longthmtitle{Expressiveness}
	Consider two risk perception models $\mathcal{M}_1$ and $\mathcal{M}_2$ with corresponding classes of risk value functions $\mathcal{V}_1$ and $\mathcal{V}_2$ with respective range spaces $\mathcal{R}_1$ and $\mathcal{R}_2$. We say that $\mathcal{M}_1$ is more expressive ($\geqq$) than $\mathcal{M}_2$ if $\mathcal{R}_2 \subseteq \mathcal{R}_1$ for any given positive random variable $\rho$. That is, $$ \mathcal{M}_1 \geqq \mathcal{M}_2 \iff \{R_2(\rho) | R_2 \in \mathcal{V}_2\} \subseteq \{R_1(\rho) | R_1 \in \mathcal{V}_1\} $$
\end{definition}

With this definition, we can compare expressivity of CPT with
Conditional Value at Risk (CVaR)~\cite{PA-FD-JME-DH:99}, also known as ``expected
shortfall'', another popular risk perception model in the financial
decision making community.  CVaR uses a single parameter $q \in [0,1)$
representing the fraction of worst case outcomes to consider for
evaluating expected risk of an uncertain cost $\rho$.  We will use
$R^{v}_{Q}$ to denote the perceived risk by CVaR model with $q=Q$.  So
a $q \approx 1$ considers the worst case outcome of $\rho$ and a $q =
0$ considers all the outcomes thus making the CVaR value equal to
expected risk ($R^v_0=R_E$).
Now we will proceed to compare expressiveness of Expected Risk, CVaR and CPT with parametrized risk value function classes $R^e$, $R^v$ and $R^c$ respectively. 

\begin{proposition}
	\label{prop:exp_risk_express}
	Let us Consider Expected risk (ER), CVaR and CPT risk models with risk value function classes $R^e$, $R^v$ and $R^c$ defined accordingly. Then, Expected Risk is the least expressive of the three models, that is $\text{CVar} \geqq \text{ER}$ and $\text{CPT} \geqq \text{ER}$ for any given random variable $\rho$.   
\end{proposition}
	\begin{proof}
		The function class $R^e$ has a single function $\mathbb{E}(\rho)$ which gives the expected value of $\rho$. So the range set of $R^e$ is a singleton, containing the expected value of $\rho$.  It is easy to see that by choosing a function $R^v_0 \in R^v$ and $R^c_{\overline{\Theta}}
		 \in R^c$ where $ \overline{\Theta}= \{1,1,1,1\} $ we have $R^v_0(\rho)=R^e(\rho)=R^c_{\overline{\Theta}}$. Which implies that $ \mathbb{E}(\rho) \in \{R^v_q(\rho) | q \in [0,1)\}$  and also $ \mathbb{E}(\rho) \in \{R^c_\Theta(\rho) | R^c_\Theta \in R^c\}$, thus proving the expressive order. 
		\end{proof}

Next we will look at the relationship between the CVaR value of a random variable $\rho$ and the expected value of another random variable $\kappa \rho$ where $\kappa > 0$ is a scaling factor.

\begin{proposition}
	Let us consider CVaR value of a given random variable $\rho$, then there exists a $\kappa_q \geq 1$ such that $R^v_q(\rho)= \mathbb{E}(\kappa_q \rho$) for all $q \in [0,1)$. 
\end{proposition}
\begin{proof}
	The range space of $R^v$ is $[\mathbb{E}(\rho), b]$, where $b$ is the worst case outcome of $\rho$. We also know that $\mathbb{E}(\kappa\rho)= \kappa \mathbb{E}(\rho)$. From this we can construct $\kappa_q=\frac{R^v_q(\rho)}{\mathbb{E}(\kappa\rho)}$ which shows that $R^v_q(\rho)= \mathbb{E}(\kappa_q\rho)$ 
	\end{proof} 

From this we can compare the expressivity of CPT and CVaR models.

\begin{proposition}
	Let us consider CVaR and CPT risk models with risk value function classes $R^v$ and $R^c$ defined accordingly. Then, we have the expressive order $\text{CPT} \geqq \text{CVaR}$ for any given random variable $\rho$.
\end{proposition} 

\begin{proof}
	Considering a subclass of CPT value functions $R^c_{\Theta*}$ where $\Theta* \in \{\Theta| \alpha=1,\beta=1,\gamma=1,\lambda=\kappa_q\}$, we have CPT value $R^c_{\theta}=\mathbb{E}(\kappa_q\rho)$ with $\theta \in \Theta*$ and some constant $\kappa_q$ defined according to previous proposition. From this we can say that $\{R^v_q(\rho) | q \in [0,1)\}$
 $\subseteq \{R^c_{\theta}| \theta \in \Theta*\}$
   $\subseteq \mathcal{R}(R^c)$, which concludes the proof.   
	\end{proof} 

\begin{lemma}
	\label{lem:CPT_expressiveness}
	Let us Consider Expected risk (ER), CVaR and CPT risk models with risk value function classes $R^e$, $R^v$ and $R^c$ defined accordingly. Then, we have the following expressive order : $\text{CPT} \geqq \text{CVaR} \geqq \text{ER} $.
\end{lemma}	

\begin{proof}
	The proof immediately follows from the above three propositions.
	\end{proof}
The above arguments imply that risk aversion can equivalently be modeled as an expected value of a scaled random variable, with greater scaling implying higher risk aversion. This is captured in both CPT ($\lambda$ parameter) and CVaR ($q$ parameter) models. Additionally, CPT also captures risk sensitivity and uncertainty sensitivity which makes it more expressive than CVaR. This can be visualized in Fig.~\ref{fig:empirical_expressiveness}. A thousand samples of $\Theta \in \real^4$ were drawn uniformly randomly for CPT, while for CVaR, $q$ was sampled uniformly across $[0,1]$. Three distributions: Normal, half Normal and Uniform, were considered with mean $100$ and standard deviation of $10$ for the first two and a range of $[70,130]$ for the latter. The median values for each box
plot is indicated on the top row. The mean value of the distribution
is indicated as ``stars'', the black lines above and below the box
represent the range, and $+$ indicates outliers. It can be clearly seen that the range of values captured by CPT is greater, which is in accordance with the theoretical argument above. Next, we will propose a method to evaluate expressiveness in the context of path planning.  

\begin{figure}
	\centering
	\includegraphics[width=\linewidth]{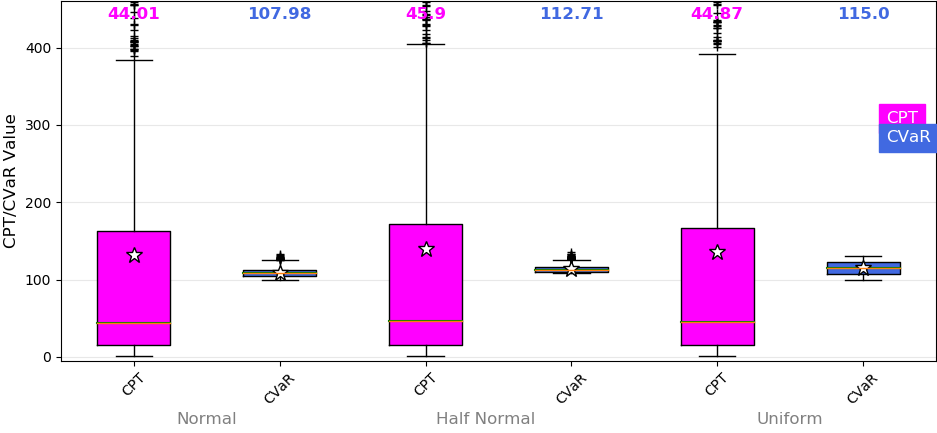}
	\caption{Boxplots showing the distribution of 1000 value samples of CPT and CVaR risk perception model for three different distributions.}
	\label{fig:empirical_expressiveness}
\end{figure}
   
\subsection{Comparing expressiveness in path planning}
%\margin{From here on, please rephrase the section to conform with the
%  previous objectives. So eliminate ``decision maker'' and
%  ``learning'' from the picture. Also change the loss function
%  notation by something else, like ``$\mathcal{L}$''}
  
%Let us suppose that a decision maker has a particular set of
%parameters $\Theta_d$, which we wish to use in the path planning
%problem, but this set $\Theta_d$ is
%unknown. % So we need to learn the parameters $\Theta_d$ before
% we can proceed to the planning phase. 
%A natural way to observe $\Theta_d$ is to let the decision maker draw
%a candidate path $P_d$ on a given environment. 
%\margin{I've erased any reference to 'drawing' a path}

\color{black}
 Let us suppose that we have an arbitrary
example path $P_d$ drawn in the environment. If the class of CPT
planners is expressive enough, we should be able to find a set of
parameters that is able to to exactly mimic this drawn path.  Since an
arbitrary path $P_d$ belongs to a very high dimensional
space\footnote{An arbitrary path can be modeled as a curve defined by
  a large number of parameters (possibly infinite). } and the planner
parameters are typically finite, any amount of parametric tuning may
not produce good approximations.  This is what we evaluate in the
following. We use the term $\Ar(P;P_d) \in
\realnonnegative$ to denote the area enclosed between the given path
$P_d$ and another path $P$. This value measures the closeness between
$P$ and $P_d$.

A path $P$ produced by a CPT planner can be represented by the CPT
parameters $\Theta$. In order to find the closest possible path $P^*$
to $P_d$ we have to evaluate
\begin{equation}
  \argmin_{P_{\Theta}, \Theta \in \mathcal{T}} \Ar(P_{\Theta};P_d),
\label{eqn:area_argmin}
\end{equation}   
where $P_{\Theta}$ is the path produced by CPT-RRT* with CPT
parameters $\Theta$, and $\mathcal{T}$ is the set of all possible
values of $\Theta$. Directly evaluating \eqref{eqn:area_argmin} is
computationally not feasible as the set $\mathcal{T}$ is infinite and
resides in $4D$ space.

%To learn these parameters, one can employ a loss function $J:
%\bm{\Theta} \rightarrow \realnonnegative $, where $\bm{\Theta}$ is the
%space for all sets of parameters $\Theta$ introduced in
%Section~\ref{sec:PT},
%% \margin{define the function with domain and range
%%  to unify notation} 
%and such that $J(\Theta)=0 \Longleftrightarrow \Theta=\Theta_d $ and
%$J(\Theta)>0\ \forall\ \Theta \neq \Theta_d$. The natural choice of
%$J$ would be the value of an optimal path joining a start and a goal
%position, which depends on the parameters $\Theta$.
An alternative to \eqref{eqn:area_argmin} is to use parameter
estimation algorithms to determine $\Theta^* \in \mathcal{T}$ which
characterizes the path $P^*$ with $\Ar(P_{\Theta};P_d)$ as a loss/cost
function. We note that neither $\Ar(P_{\Theta};P_d)$ can be computed
directly (without running CPT-RRT* first), nor the gradient of $\Ar$
wrt $\Theta$ is accessible. This limits the use of standard gradient
descent algorithms to estimate $\Theta^*$.  To address this problem,
we use SPSA~\cite{JCS:03} with $\Ar(P_{\Theta};P_d)$ as the loss
function to estimate the parameters $\Theta^*$.  Next, we briefly
explain the main idea and adaptation of SPSA to our setting and refer
the reader to~\cite{JCS:03}
%\margin{use tildes}
for more detailed treatment and analysis of the SPSA algorithm.

We start with an initial estimate $\Theta_0$ and iterate to produce
estimates $\Theta_k$, $k \in \integernonnegative$ using the loss
function measurements $\Ar(P_{\Theta_k};P_d)$.  The main idea is to
perturb the estimate $\Theta_k$ according to ~\cite{JCS:03} to get
$\Theta^+$ and $\Theta^-$, for the $\supscr{k}{th}$ iteration. These
perturbations are then used to generate the perturbed paths
$P_{\Theta^+}, P_{\Theta^-}$ using Algorithm~\ref{alg:hi-rrt*}. With
these perturbed paths, the loss function measurements
$\Ar(P_{\Theta^+};P_d),\Ar(P_{\Theta^-};P_d)$ are evaluated and used to
update our parameter $\Theta_k$ according to~\cite{JCS:03}. To test
the goodness of the updated parameter, we determine the corresponding
path $P_{\Theta_{k+1}}$ and measure $\Ar(P_{\Theta_{k+1}};P_d)$. If the
area is within a tolerance $\kappa \in \realpositive$, that is, if
$\Ar(P_{\Theta_{k+1}};P_d) < \kappa$, the iteration stops and
$P_{\Theta_{k+1}}$ is returned. We followed the guidelines from
\cite{JCS:03} for choosing the parameters used in SPSA.
%$\Delta_k,a_k,c_k$
%with $ a_k,c_k\ > 0\ ; \ a_k,c_k\ \rightarrow 0\ \text{as}\ k
%\rightarrow \infty\ $ and $\Delta_k$ is symmetrically distributed
%around 0 and $E(\Delta_k^2),E(\Delta_k^2) < \infty$. 
%\margin{explain these guidelines here (hopefully
%  there is something special to our problem case)} 
The results of this adaptation are evaluated and compared with the
results that employ other risk perception models in
Section~\ref{sec:results}.

\section{Results and Discussion}
\label{sec:results}
\color{black}
In this section 
%\margin{when ``Section'' does not have a label, just
%  use ``section''} 
  we illustrate the results of the solutions to the problem statement proposed in 
Sections~\ref{sec:CPT_environment},\ref{sec:planning}, and~\ref{sec:learning}
%\margin{these are ``Sections''}
% Firstly, we will use the environment described
%in Section~\ref{sec:CPT_environment}  shown in
%Figure~\ref{fig:intro_env_perception}, to illustrate planning in a general environment having continuous risk and uncertainty. Then, we will also illustrate the environment generation and planning 
considering a specific scenario having some risk and uncertainty profiles.

\subsection{Environment Perception and Planning}
\label{sec:results_env_plan}
We consider a hypothetical scenario where an agent needs to navigate
in a room during a fire emergency. In this, the 2D configuration space
for planning becomes $\mathcal{C} =~[-10,10] \times [-10,10]$.  The
agent is shown a rough floor map (Figure~\ref{fig:fire_env}) with
obstacles (which are thought to be ablaze) in the environment with a
blot of ink/torn patch, making that region unclear and hard to
decipher. This results in the spatial uncertain cost $\rho$ with first
moment ($\rho_\mu$) represented by cost associated to obstacles and
fire source and second moment ($\rho_\sigma$) represented by the
uncertainty associated to the ink spot/tear.

%The
%agent is provided with a rough floor map with uncertain measurements
%of the obstacles (which are thought to be ablaze) in the environment
%(source of risk), also having a blot of ink/ torn patch, making that
%region unclear and hard to decipher (source of
%uncertainty). \margin{here we need to be careful on how we phrase
%  ``source of risk'' and ``source of uncertainty''. Probably all of it
%  can be termed as uncertain cost. The difference between these terms
%  is not clear here. } This is illustrated in
%Figure~\ref{fig:fire_env}.

The blue colored objects are the obstacles whose location is known to
be within some tolerance (dark green borders) and the light orange 
%\margin{it
%  seems they are yellow} 
ellipses illustrate that these objects have
caught fire. The grey ellipse indicates a possible tear/ink spot on
the map, which makes that particular region hard to read. The start
and goal positions are indicated as blue spot and green cross
respectively.  We use a scaled sum of bi-variate Gaussian distribution
to model the sources of continuous cost (orange ellipses) with
appropriate means and variances to depict the scenario in
Figure~\ref{fig:fire_env}. We utilize bump functions from differential
geometry to create smooth ``bumps'' depicting the
discrete obstacles. One approach to do this is described in
Section~\ref{sec:environment_setup}. An alternative procedure is
briefly described as follows. Consider the maximum cost value imparted
to the obstacles as $\supscr{\rho}{max} \in \realnonnegative$ and let
$a_1,a_2,b_1,b_2 \in \realnonnegative$ be the inner (blue rectangle)
and outer (dark green borders) measurements of the obstacles from the
center $c = (c_1,c_2) \in \mathcal{C}$. Let $x = (x_1,x_2) \in
\mathcal{C}$ be a point in the configuration space with $f,g,h$ being
real valued scalar functions given by $f(y)= e^{-\frac{1}{y}}, y \in
\realpositive\ \text{and}\ f(y)= 0\ \text{otherwise} $, $g(y) =
\frac{f(y)}{f(y)+f(1-y)}$ and
$h(y)=1-g(\frac{y^2-a^2}{b^2-a^2})$. Then, $\rho_\mu(x)$ can be
calculated by :
\begin{equation}
\rho_\mu(x)= \supscr{\rho}{max}h(x_1-c_1)h(x_2-c_2).
\label{eqn:bump_fn}
\end{equation}
This procedure produces smooth ``bumps'' in the cost profile which are
visualized in Figure~\ref{fig:intro_mean_risk} using
$\supscr{\rho}{max}=20$. This approach can be easily generalized to
arbitrary high dimensions by simply multiplying upto $h(x_i-c_i)$
terms in \eqref{eqn:bump_fn} to create a bump function in the
$\supscr{i}{th}$ dimension.  To generate the second moment of cost $\rho_{\sigma}$, we
use a scaled bi-variate Gaussian distribution with appropriate means
and variances to depict the ink spot/tear in
Figure~\ref{fig:fire_env}. Now we will illustrate the results of implementing Algorithm~\ref{alg:hi-rrt*} in this environment.

\paragraph*{Simulations and discussions}
\label{sec:results_env_plan_sims}
\label{sec:results_env_plan_discussion}
With the uncertain cost $\rho$ with moments $\rho_\mu$ and $\rho_\sigma$
from previous paragraph, we use a half Normal distribution and
discretization factor $M=20$ to generate the costs $\rho(x)$ and their
corresponding $p(x)$ from Section~\ref{sec:CPT_environment}, the
results of using Algorithm~\ref{alg:CPT_env} to every point in $\mathcal{C}$ to generate the perceived
environment is shown in Figures~\ref{fig:intro_plan} and
\ref{fig:intro_env_perception}. The level of risk at a point $R^c$ or $R^e$ is indicated by color map. Figure~\ref{fig:intro_plan_er} shows
a rationally perceived environment using expected risk $R^e$. Whereas,
Figure~\ref{fig:intro_plan_cpt} indicates a non-rational highly risk
averse perception using CPT ($R^c$) with $\Theta=\{0.74,2,0.9,10\}$ having a
high $\lambda$ value. A risk indifferent profile
(Figure~\ref{fig:intro_CPT_risk2}) is generated by
$\Theta=\{0.74,1,0.3,2.25\}$ having a low risk sensitivity $\gamma$
value. Similarly, uncertainty indifferent profile
(Figure~\ref{fig:intro_CPT_unc1}) and uncertainty averse profile
(Figure~\ref{fig:intro_CPT_unc2}) are generated by fixing $\alpha$ and
having high and low $\beta$ values respectively.
\begin{figure*}
\begin{subfigure}[t]{0.09\linewidth}\centering
        \includegraphics[width=\textwidth]{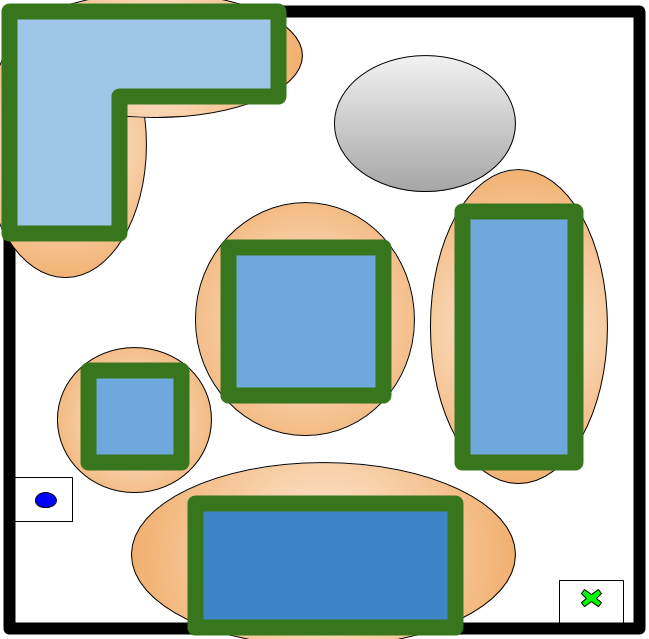}
        \caption{Rough sketch of environment.}
        \label{fig:fire_env}
    \end{subfigure}
    ~
    \begin{subfigure}[t]{0.17\linewidth}
        \includegraphics[width=\textwidth]{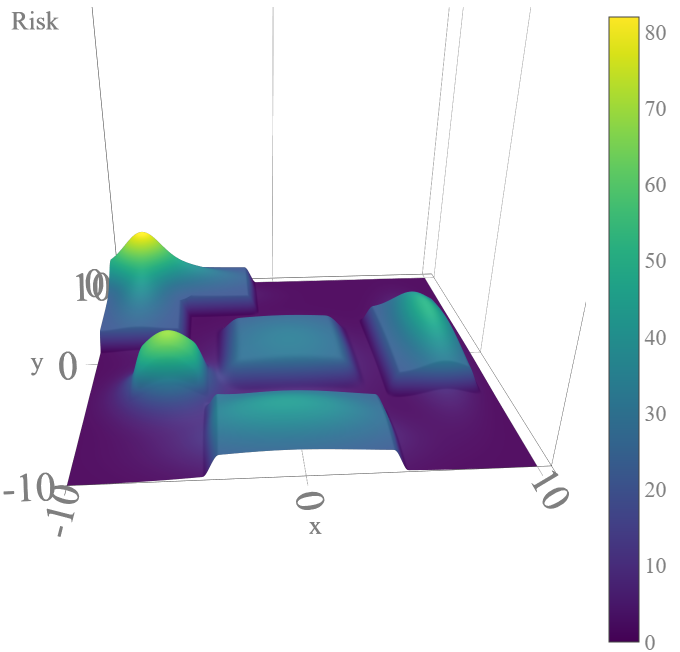}
        \caption{Mean cost $\rho_\mu$}
        \label{fig:intro_mean_risk}
    \end{subfigure}%
    ~
    \begin{subfigure}[t]{0.17\linewidth}
        \includegraphics[width=\textwidth]{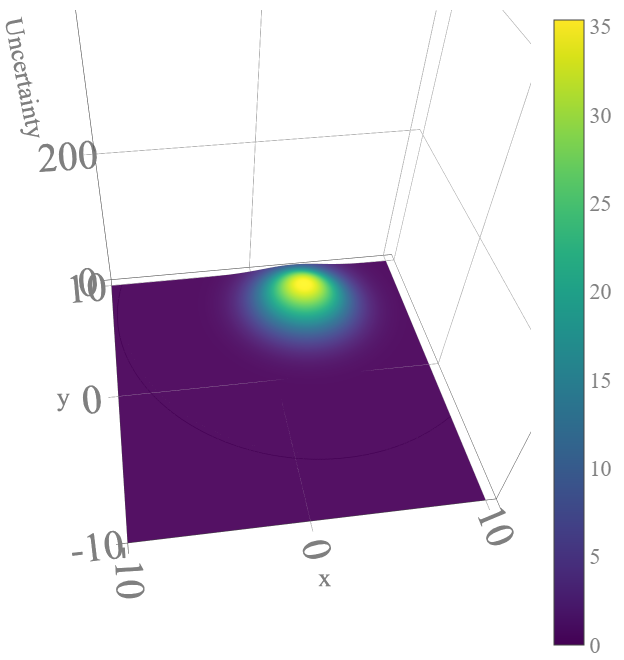}
        \caption{Uncertainty $\rho_\sigma$}
        \label{fig:intro_variance}
    \end{subfigure}%
    ~
%    \begin{subfigure}[t]{0.19\linewidth}
%        \includegraphics[width=\textwidth]{exp_risk3D_eps2_1_max_val20.png}
%        \centering
%        \caption{Expected risk profile}
%        \label{fig:intro_exp_risk}
%    \end{subfigure}%
%    ~
%    \begin{subfigure}[t]{0.19\linewidth}
%        \includegraphics[width=\textwidth]{CPT_risk3D.png}
%        \caption{CPT risk averse profile}
%        \label{fig:intro_CPT_risk}
%    \end{subfigure}%
    \begin{subfigure}[t]{0.17\linewidth}
        \includegraphics[width=\textwidth]{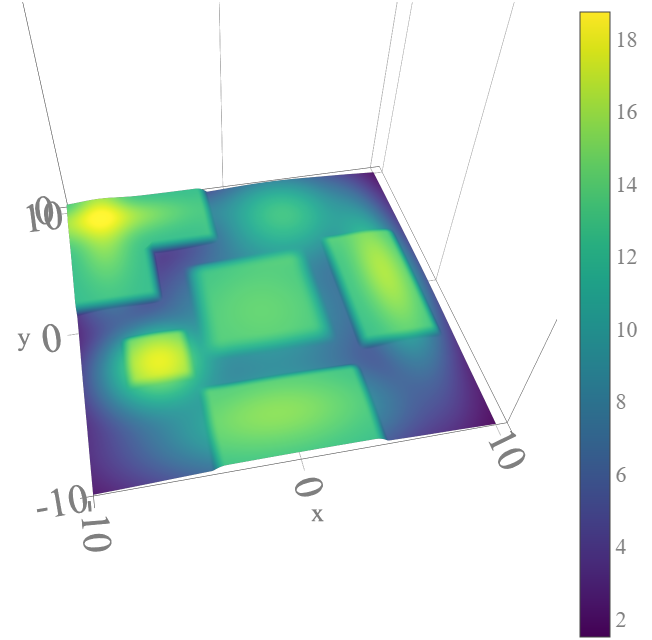}
        \centering
        \caption{Risk indifferent profile with $\Theta=\{0.74,1,0.3,2.25\}$}
        \label{fig:intro_CPT_risk2}
    \end{subfigure}
    ~
    \begin{subfigure}[t]{0.17\linewidth}
        \includegraphics[width=\textwidth]{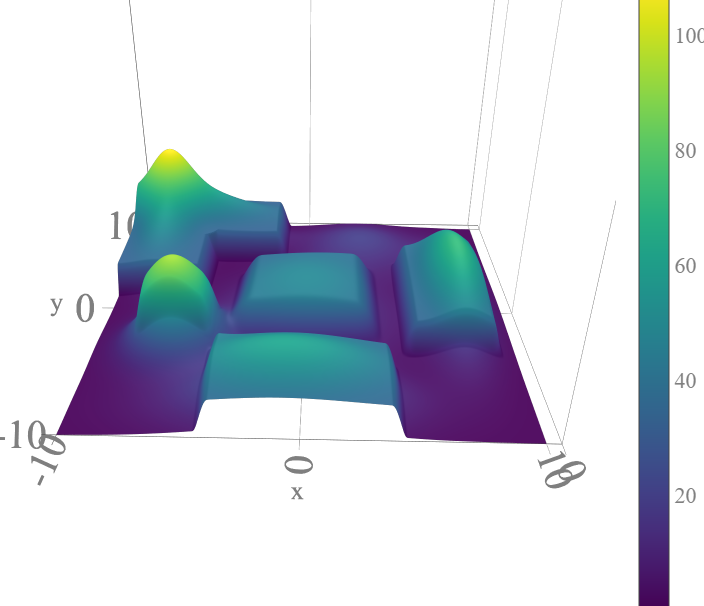}
        %\centering
        \caption{Uncertainty~indifferent profile~with $\Theta~=~\{0.74,3,0.88,2.25\}$}
        \label{fig:intro_CPT_unc1}
    \end{subfigure}
    ~
    \begin{subfigure}[t]{0.17\linewidth}
        \includegraphics[width=\textwidth]{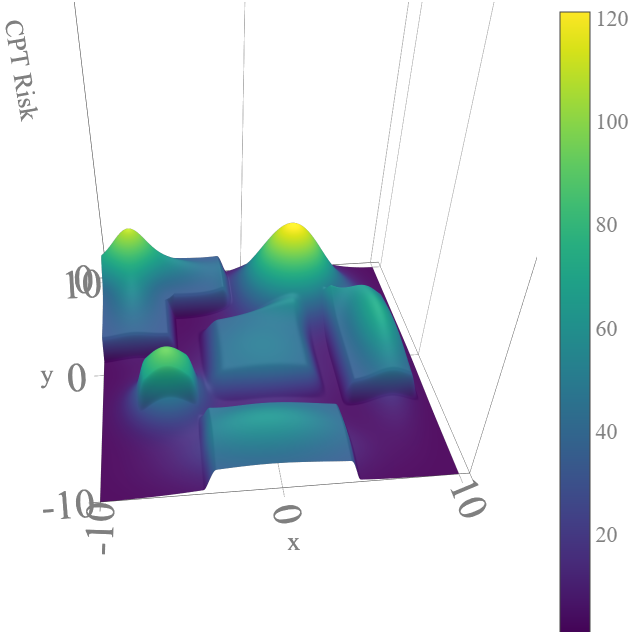}
        %\centering
        \caption{Uncertainty averse profile $\Theta=\{0.74,0.05,0.88,2.25\}$}
        \label{fig:intro_CPT_unc2}
    \end{subfigure}
    \caption[Caption]{Environment perception using CPT.}
    \label{fig:intro_env_perception}
\end{figure*}

After the perceived environment is generated,
Algorithm~\ref{alg:hi-rrt*} is used to plan a path from the start
point to the goal point shown in Figure~\ref{fig:fire_env}. We use
$T=20,000$ iterations for the CPT-RRT* algorithm with $\delta=10^{-4}$. The
same random seed was used for all executions for consistency. The path
planning results are illustrated in Figure~\ref{fig:path_changes}. As
expected, we see that the path depends on the perceived risk profile. Figure~\ref{fig:fire_env_cpt_risk_averse} indicates a circuitous
path due to the highly risk averse perception, whereas
Figure~\ref{fig:fire_env_exp} indicates a shorter and more direct path
for a rational DM using expected risk. Increasing the uncertainty
sensitivity (lowering $\beta$) and reducing risk aversion (lowering $\lambda$) makes the
planner avoid the highly uncertain ink spot/tear in the top-right
region and take a more riskier path in the lower region as shown in
Figure~\ref{fig:fire_env_cpt}. By having a medium risk aversion and lower uncertainty
sensitivity (increasing $\beta$), the planner produces a different path through the medium
risky and uncertain middle region as shown in
Figure~\ref{fig:fire_env_cpt_unc}.
%Algorithm~\ref{alg:hi-rrt*} are included here using the setup
%described in the previous paragraph with fixed $x_s=[0,8]$ and
%$x_g=[-2,-8]$. First, we fix $\lambda=2.25$, then we vary $\gamma$ as
%illustrated in Figure~\ref{fig:gamma_changes}, the CPT-RRT* tree is
%indicated as white lines and the final path produced is indicated in
%red. We see that the behavior is as expected, and with decrease in
%$\gamma$, the risk sensitivity decreases. We see that, in the cases where risk sensitivity is low
%when $\gamma=0.28,0.18$, the CPT-RRT* paths is almost a straight line
%indicating the indifference towards the risk profile of the
%DM. Whereas, in cases when $\gamma=0.88$ we see that the CPT-RRT*
%Algorithm avoids the high risk area in the center, indicating
%sensitivity towards risk.

\begin{figure}
%\centering

    \begin{subfigure}[t]{0.49\linewidth}\centering
        \includegraphics[width=\textwidth]{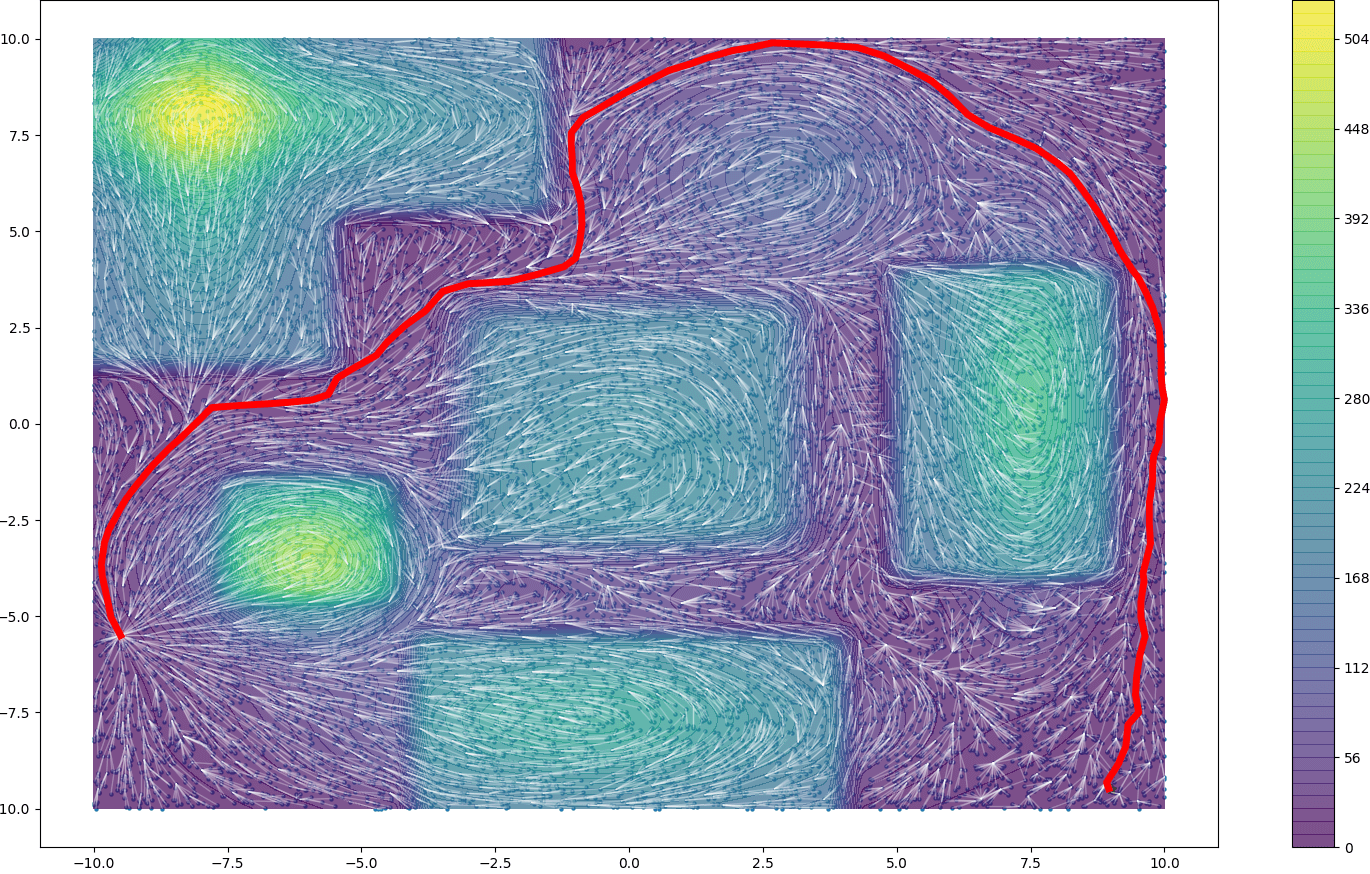}
        \caption{High cost and high uncertainty sensitivity ($\Theta=\{0.74,1,0.9,10\}$)}
        \label{fig:fire_env_cpt_risk_averse}
    \end{subfigure}%
    ~ 
    \begin{subfigure}[t]{0.49\linewidth}\centering
        \includegraphics[width=\textwidth]{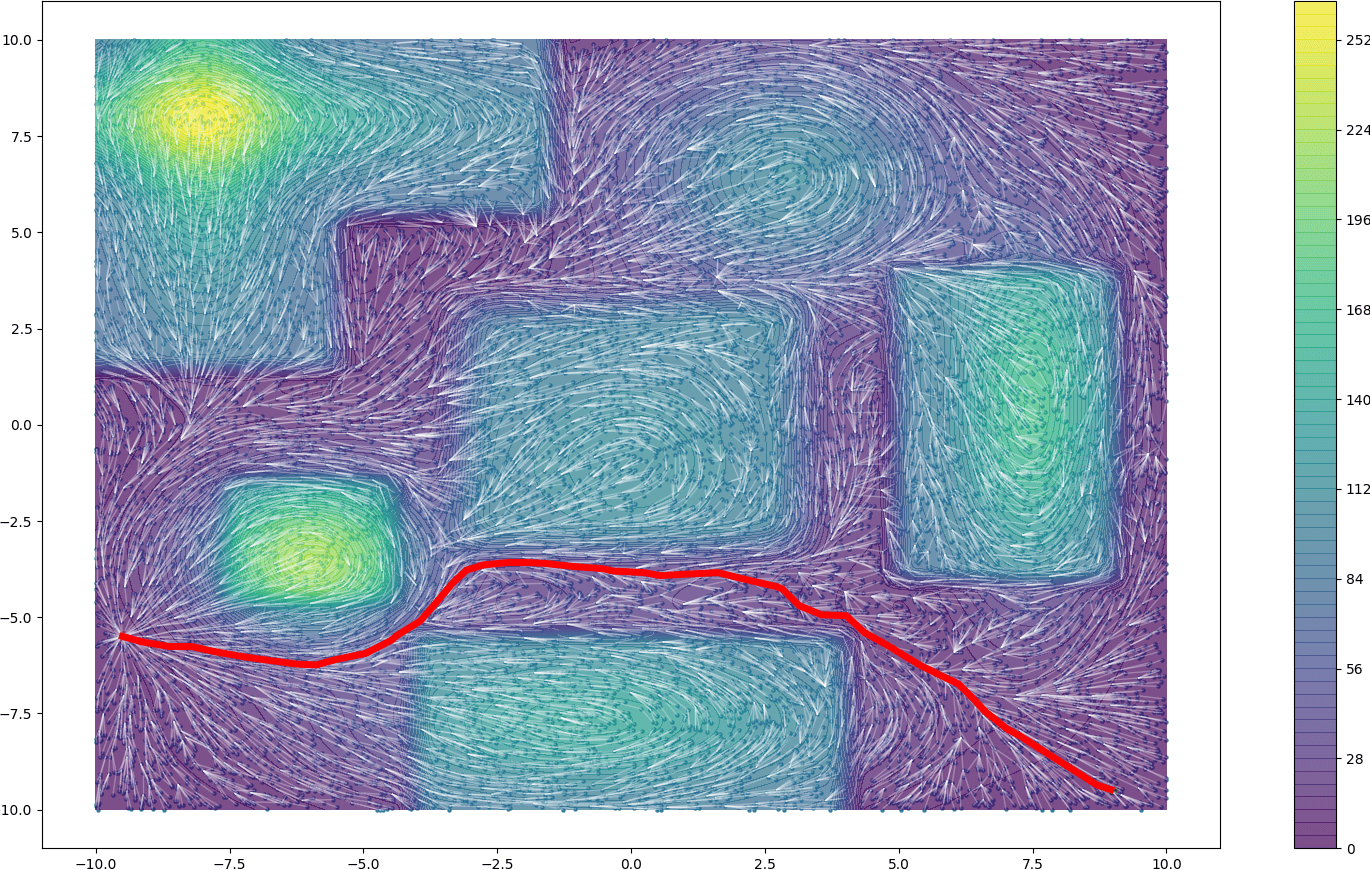}
        \centering
        \caption{Medium cost aversion and high uncertainty sensitivity ($\Theta=\{0.74,1,0.9,5\}$)}
        \label{fig:fire_env_cpt}
    \end{subfigure}
    
    \begin{subfigure}[t]{0.49\linewidth}\centering
        \includegraphics[width=\textwidth]{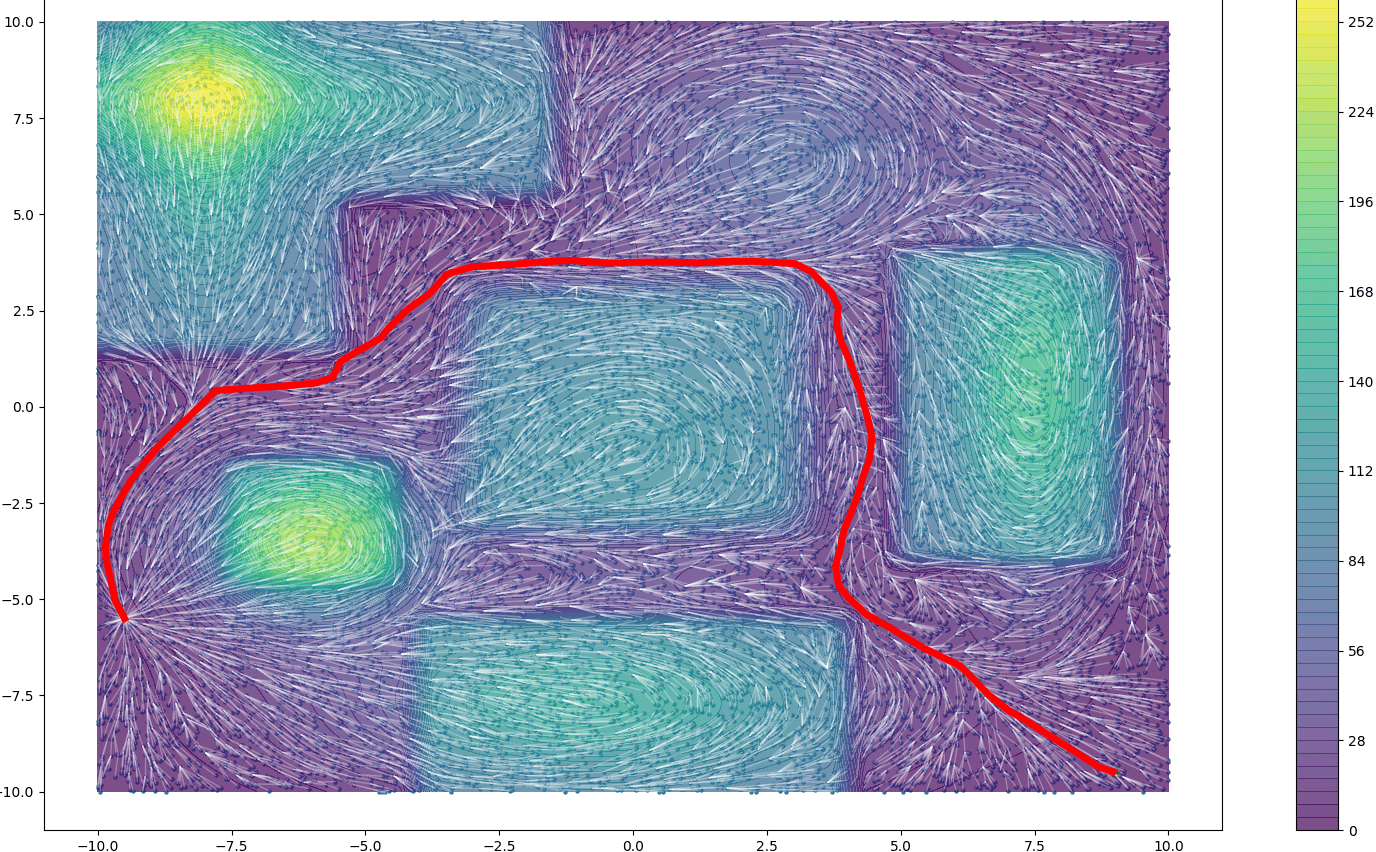}
        \caption{Medium cost aversion and low uncertainty sensitivity ($\Theta=\{0.74,2,0.9,5\}$)}
        \label{fig:fire_env_cpt_unc}
    \end{subfigure}%
    ~ 
    \begin{subfigure}[t]{0.49\linewidth}\centering
        \includegraphics[width=\textwidth]{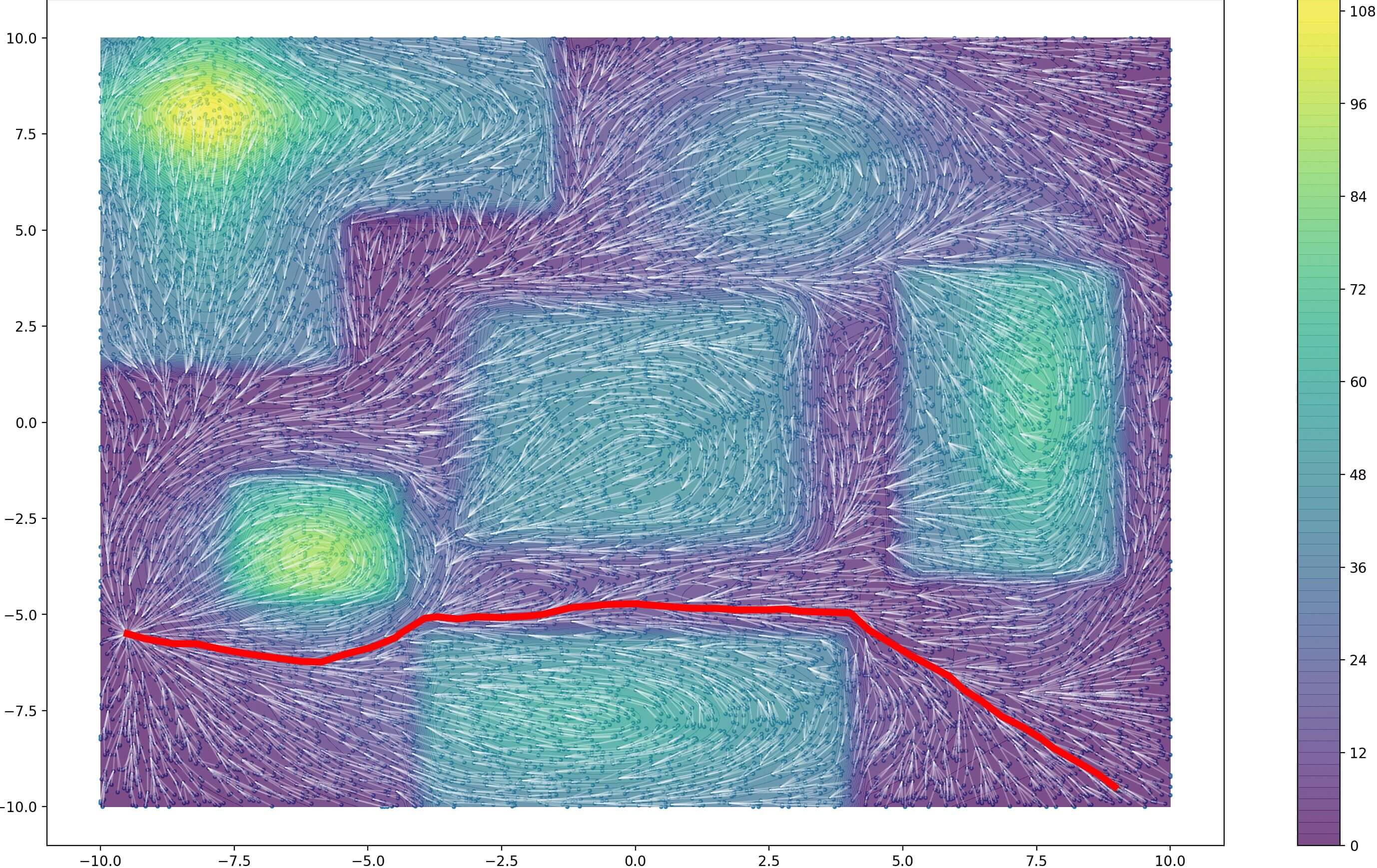}
        \centering
        \caption{Expected Risk}
        \label{fig:fire_env_exp}
    \end{subfigure}
    \caption[Caption]{Paths produced by CPT-RRT* under different
      perception models. White lines indicate the tree grown from the
      start position, red line indicates the optimal path to goal
      after $T=20,000$ iterations. Background color map depicts the
      CPT costs in (a)-(c) and expected costs in (d)}
    \label{fig:path_changes}
\end{figure} 

\color{black}
We also demonstrate the capability of our planner in 3D space, which is illustrated in Fig.~\ref{fig:3d_plan}. We employ a cubic configuration space measuring $10$ units, cluttered with randomly placed $50$ cube obstacles of unit volume. The start position is $(1,1,1)$ and goal position is $(9,9,9)$. There is also a continuous source of risk (modeled as a scaled normal distribution as before) centered at $(3,3,3)$. We use $\Theta=\{0.74,1,0.9,5\}$ (same as~\ref{fig:fire_env_cpt}) with $\gamma_{RRT*}=500$, $d=0.2$ and $\delta=10^{-6}$. From Fig.~\ref{fig:3d_plan}, we see that CPT-RRT is able to find a reasonably smooth path avoiding obstacles and the risky area, in similar number of iterations as in the 2D case. This shows the capability of our planner in 3D space. 

\begin{figure}
	\centering
	\includegraphics[width=0.8\linewidth]{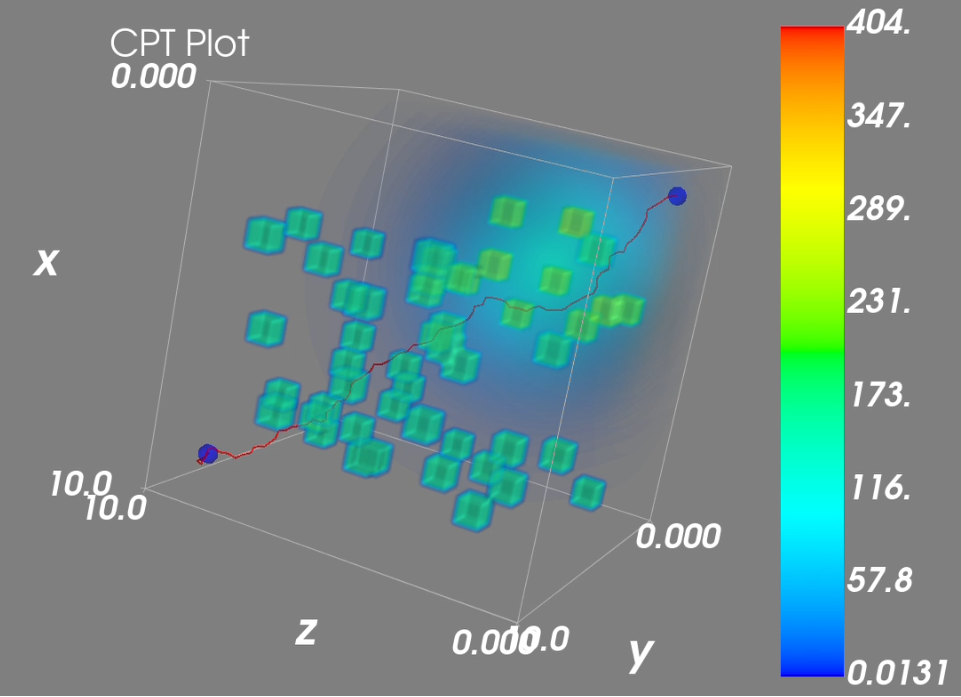}
	\caption{Path(Red line) by CPT-RRT* in 3D environment with 50 discrete random obstacles and a continuous risk source after 20,000 iterations.}
	\label{fig:3d_plan}
\end{figure}

%We have showed that CPT parameters $\Theta$ can be tuned to produce paths describing
%varying risk and uncertainty profiles.

%\margin{this sentence is incomplete}
%\color{black}
\color{black}
\paragraph*{Solution quality}Figure~\ref{fig:convergence} illustrates the empirical convergence and solution quality of the paths produced by our algorithm.
We performed empirical convergence tests, by running CPT-RRT* $100$ times with the same parameters and initial conditions and measuring the area between paths produced after every $500$ iterations for a total of $80,000$ iterations. The results are shown in Figure~\ref{fig:convergence_1}. We see that initially ($<10000$ iterations) there are changes in the output path as the space is being explored and the output path is changing. After $10,000$ iterations we consistently see minimal path changes indicating that the algorithm is converging towards a desirable path.
Then we also checked the solution quality of the path by computing the cost of the output path every $250$ iterations as shown in Figure~\ref{fig:convergence_2} for $100$ trials consisting of $25,000$ iteration. We see that the there is a consistent decrease in path cost in all the trials throughout. We also note that after $10,000$ iterations the cost decrease starts to plateau, indicating that the algorithm is close to a high quality (low cost) solution. 
From these observations of Figure~\ref{fig:convergence}, we recommend upwards of $T=10,000$ iterations to achieve
smooth and consistent paths in our setting.

\begin{figure}
\centering

\begin{subfigure}[t]{0.49\linewidth}\centering
        \includegraphics[width=\textwidth,height=1.2in]{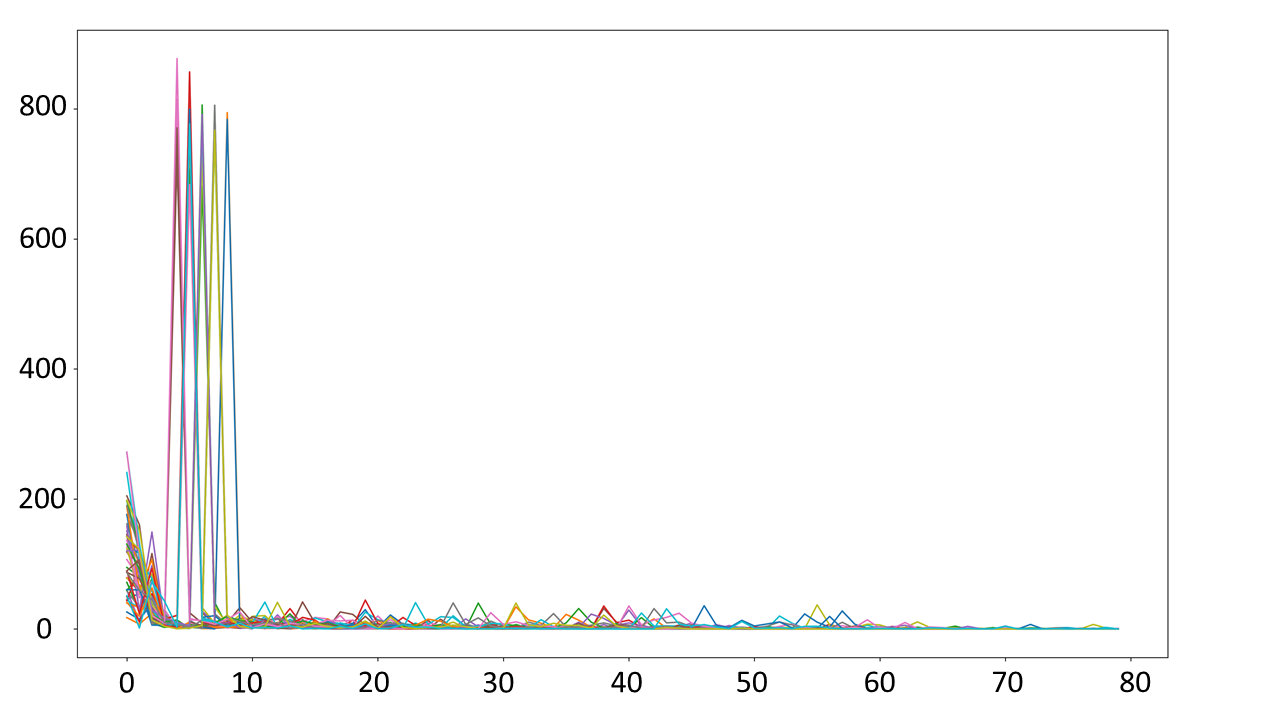}
        \caption{Convergence over iterations}
        \label{fig:convergence_1}
    \end{subfigure}%
    ~ 
    \begin{subfigure}[t]{0.49\linewidth}\centering
        \includegraphics[width=\textwidth,height=1.2in]{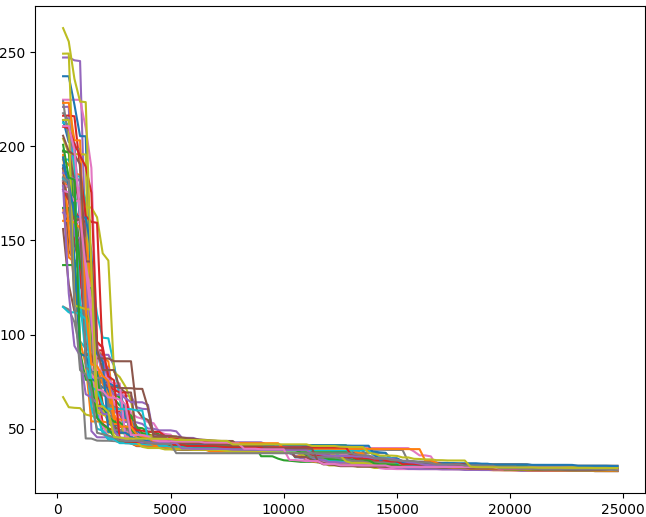}
        \centering
        \caption{Path cost over iterations}
        \label{fig:convergence_2}
    \end{subfigure}
\caption[Caption]{a) Empirical convergence analysis. The distance between paths after every $500$ iterations (y-axis) with the number of iterations in thousand (x-axis). b) Cost of the output path (y-axis) every $250$ iterations with the number of iterations (x-axis) }
\label{fig:convergence}
\end{figure} 

\color{black}

% We note the
%path smoothness and consistency is enhanced with higher iterations $T$, due to
%asymptotic properties of CPT-RRT*
%\margin{be careful with what algorithm you
 % refer to}
%algorithm.  
%We also note, with large iterations, collision avoidance can be achieved by choosing a $\supscr{\rho}{max}$ greater than the continuous risk sources and having a sufficiently high risk sensitivity $\gamma$. This can be observed in the simulations shown in Figure~\ref{fig:path_changes}.
%We also note that the iterations are
%faster compared to a traditional RRT* pipeline as there is no obstacle
%checking involved, which saves significant time in complex spaces.
%From Figures~\ref{fig:gamma_changes},~\ref{fig:lambda_changes}, it can
%be noted that flatter sections (areas with almost constant risk profiles) in the environment can produce similar
%paths for different $\Theta$, so non-flat environments 
%%\margin{what is
%  %a flat environment?}
%  are needed to capture the essence of nonlinear
%perception among DM. Future work will be devoted to the analysis of
%how non-flat environments lead to a clearer distinction of paths with
%varying the CPT parameters.
\begin{figure}
	\centering
	\begin{subfigure}[t]{0.49\linewidth}\centering
		\includegraphics[width=\textwidth,height=1.2in]{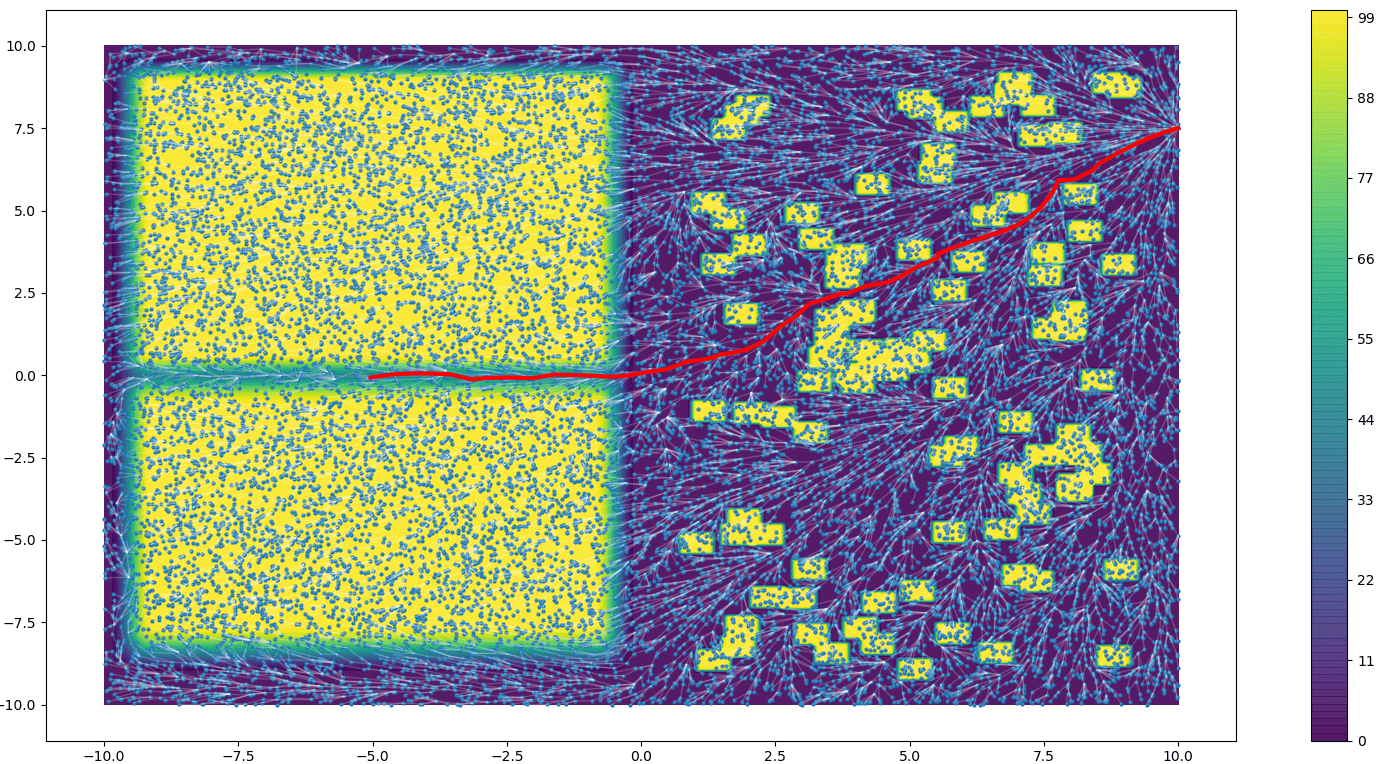}
		\caption[Caption]{Paths produced by CPT-RRT*.}
		\label{fig:cluttered_path1}
	\end{subfigure}%
	~
	\begin{subfigure}[t]{0.49\linewidth}\centering
		\includegraphics[width=\textwidth,height=1.2in]{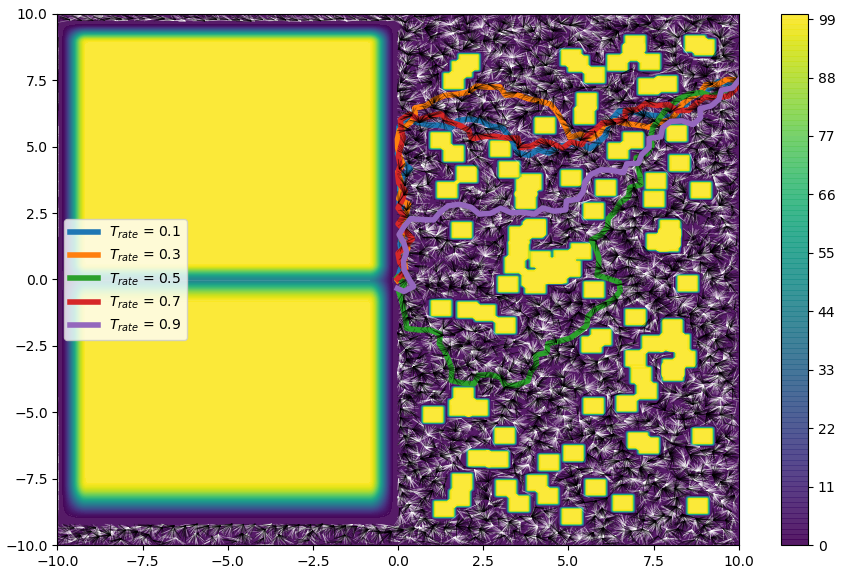}
		\caption[Caption]{Paths produced by T-RRT* using IC.}
		\label{fig:cluttered_pathIC}
	\end{subfigure}%
	 
	\begin{subfigure}[t]{0.49\linewidth}\centering
		\includegraphics[width=\textwidth,height=1.2in]{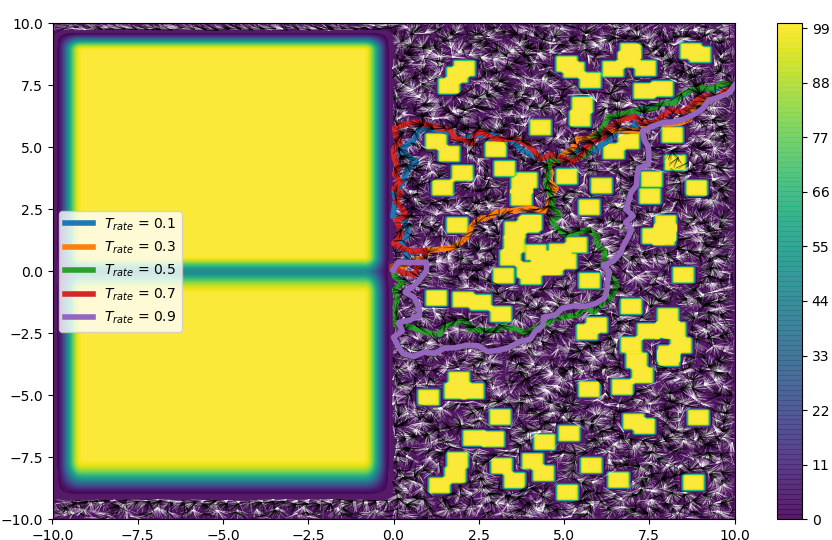}
		\caption[Caption]{Paths produced by T-RRT* using MW.}
		\label{fig:cluttered_pathMW}
	\end{subfigure}%
	~
	\begin{subfigure}[t]{0.49\linewidth}\centering
		\includegraphics[width=\textwidth,height=1.2in]{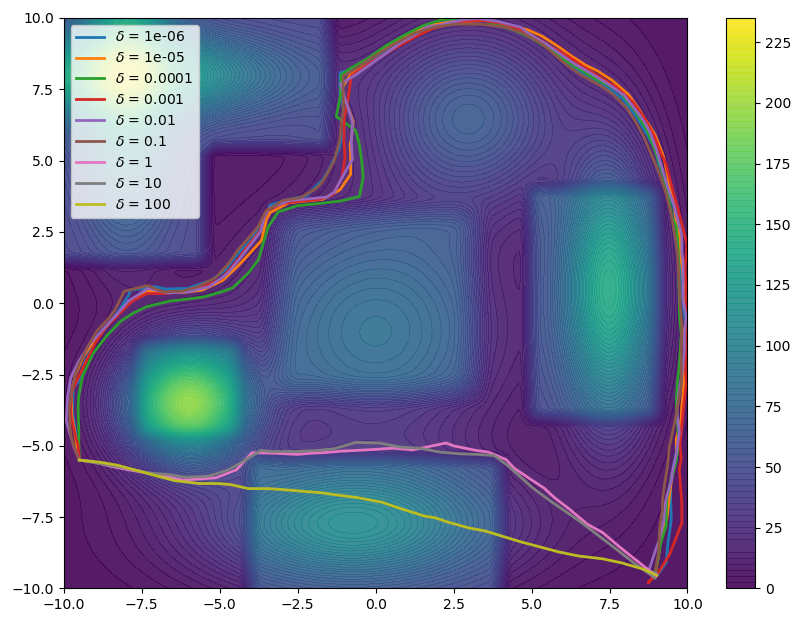}
		\caption{Paths produced by varying $\delta$}
		\label{fig:delta_paths}
	\end{subfigure}
	\caption[Caption]{Path comparison with T-RRT* and varying $\delta$. (a)-(c)Paths produced in a cluttered environment using $T=20,000$ iterations for CPT-RRT* and $20,000$ nodes for T-RRT*\footnote{We used nodes instead of iterations for T-RRT* to maintain an equal number of nodes in the tree, as a node does not get added if it fails the transition test}. (d) Paths produced by CPT-RRT* by varying $\delta$ with $\Theta=\{ 0.74,1,0.9,7.5\}$.}
	\label{fig:result_cluterred_comparison}
\end{figure}

\paragraph*{Comparison in narrow and cluttered environments} Here, we
will illustrate and compare the performance between our RRT* framework
and T-RRT*~\cite{DD-TS-JC:16} (another algorithm operating on
continuous cost spaces) in a cluttered environment with narrow
passages as shown in Figure~\ref{fig:result_cluterred_comparison}. To
construct this environment, we used 100 randomly placed small objects
on the right half and two big objects separated by a narrow passage on
the left half. Start point $x_s$ is on the top right corner and the
goal $x_g$ is at the center of the narrow passage. Bump functions
similar to previous paragraphs were used to construct a smooth spatial
cost $\rho$ from the obstacles. Since T-RRT* does not have risk
perception capabilities, for a fair comparison we use the continuous
cost $\rho$ to implement both algorithms. In this way, we will be able
to specifically compare the planning capabilities of both algorithms
in the same continuous cost environment. We used
$\subscr{\gamma}{RRT*}=100$ and $d=0.35$ for both algorithms. From
Figure~\ref{fig:cluttered_path1} we can see that our algorithm is able
to sample and generate paths in the narrow passage, as well as avoid
obstacles in a cluttered environment. In comparison, we can see that
from T-RRT* employing integral cost (IC) in
Figure~\ref{fig:cluttered_pathIC} and minimum work (MW) in
Figure~\ref{fig:cluttered_pathMW} cannot generate paths in the narrow
costly region fast enough irrespective of the $\subscr{T}{Rate}$ used
due to the sampling bias towards low-cost regions. Also, T-RRT* paths
do not appear to be as smooth as the paths from our framework,
irrespective of the cost(IC or MW) used. We also note that, the
cluttered and high cost environment induces a high failure rate of the
transition test, resulting in longer run times of T-RRT* required to
build the same number of nodes as our algorithm, especially for high
$T_{Rate}$ values.

\paragraph*{Comparison in dynamic environments}
\color{black} Here, we contrast the performance of 
CPT-RRT* and Risk-RRT*~\cite{WC-MQM:17} (a risk aware planner)
in a $10$ by $10$ environment area with static and moving obstacles as
shown in Figure~\ref{fig:result_dynamic_comparison}. To account for
risk dynamics, we will be planning in the space-time domain, and
assume knowledge of the dynamics (or a good estimate) of
$\rho_{\mu}(t)$ and $\rho_{\sigma}(t)$, which will result in a
time-varying perceived risk map $R^c(x,t)$. We also assume that each
edge in the tree will be traversed in some time $\Delta t$. The
underlying RRT* parameters employed by both algorithms were taken to
be identical, with $\gamma_{RRT*}=100$ and $d=0.25$, while
$\delta=0.1$ for CPT-RRT*. {\color{black} Our starting point is the
	same parametric CCR Risk Map~\cite{WC-HK-YT-AY-HA-MQM:16} as in
	Risk-RRT*, which generates a continuous, and time-varying, cost map
	based on the pose and velocity of a moving human as shown in
	Fig.~\ref{fig:CCR_map} (a snapshot). The human obstacles move back
	and forth within the indicated range (gray line), with top two
	obstacles moving $d$ units in $\Delta t$ time, while middle left
	obstacle moves at $0.1d$ units. }Since the CCR map does not
incorporate uncertainty, we will use it as the mean cost
$\rho_{\mu}(t)$. We employ a scaled normal distribution on top of each
source of dynamic risk (moving human) to denote $\rho_{\sigma}(t)$,
representing uncertainty for each source. From
$\rho_{\mu}(t)$ and $\rho_{\sigma}(t)$, we calculate $R^c(x,t)$
according to Algorithm~\ref{alg:CPT_env} which is visualized in
Fig.~\ref{fig:CCR_CPT_map}. {\color{black} Note that with higher risk (moving obstacles as compared to static or no obstacles),
	the lighter the color in the map. We compare our algorithm with Risk-RRT* in two
	scenarios.} At first, for fair comparison, and since Risk-RRT* does
not consider uncertainty or risk perception models, we use directly
the CCR cost map or $\rho_{\mu}(t)$ for planning. {\color{black}
	This corresponds to a rational DM model.} The results
are summarized in Fig.~\ref{fig:box_plots_CCR} which show path length
and cumulative risk returned by CPT-RRT* and Risk-RRT* using the CCR
map in 50 trials. {\color{black} In general, we see a lower performance
	in Risk-RRT* due to its conservative approach in dealing with
	risk. }First, since risk is not explicitly accounted for in the cost
function of Risk-RRT* and ``risk'' is treated as an ``obstacle'' to
avoid, the {\color{black} resulting} path produced by Risk-RRT* is
longer, even though its cost function optimizes path
length. {\color{black} The length of the path from our planner is
	shorter, with comparable cumulative risk of the output paths of both
	planners computed as in~\eqref{eqn:cpt_costs}}. Furthermore, $12$ out of the $50$ trials in case of
Risk-RRT* could not find a solution within 15,000 iterations. This
{\color{black} seems to be a consequence of a higher sample rejection
	rate due to the tight free spaces 
	created by the dynamic obstacles when close to other objects.}
{\color{black}This drawback is more pronounced when considering a
	risk-averse DM.}  {\color{black}
	Fig.~\ref{fig:CCR_CPT_map} represents such DM who perceives that
	getting close to the dynamic obstacles is highly risky, as compared
	to the perception of a rational DM represented by
	Fig.~\ref{fig:CCR_map}. In this way, the risk values in
	Fig.~\ref{fig:CCR_CPT_map} are in the range $0-421$, which is much
	higher than those of the CCR Map in Fig.~\ref{fig:CCR_map} (with
	ranges $0-100$).}  Due to higher risk values as given by this map,
the sample rejection in Risk-RRT* is very high and could not find a
feasible path in any of $50$ trials, whereas CPT-RRT* consistently
found a path similar to the one shown in Fig.~\ref{fig:CCR_CPT_map} in
all of the 50 trials.  \color{black}
\begin{figure*}
	\centering
	\begin{subfigure}[t]{0.25\linewidth}\centering
		\centering
		\includegraphics[width=\textwidth,height=1.15in]{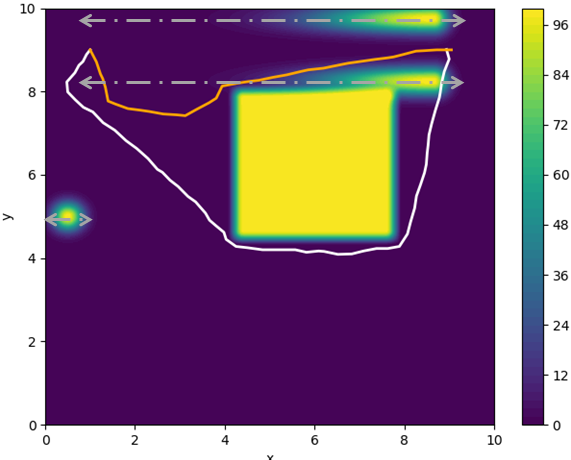}
		\caption[Caption]{Paths returned by CPT-RRT*(Orange) and Risk-RRT*(White) in CCR map}
		\label{fig:CCR_map}
	\end{subfigure}%
	~ 
	\begin{subfigure}[t]{0.45\linewidth}\centering
		\centering
		\includegraphics[width=\textwidth,height=1.15in]{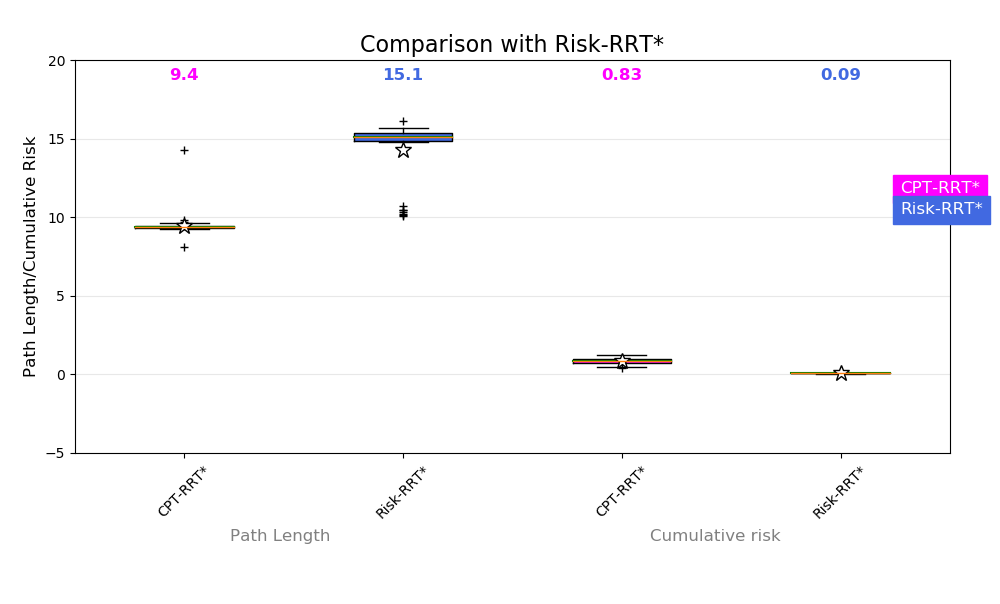}
		\caption{Boxplots showing path length and cumulative risk returned by CPT-RRT* and Risk-RRT* using the CCR map in 50 trials (38 trials for Risk-RRT* as 12 did not succeed in finding a path) consisting of 15,000 RRT* iterations. }
		\label{fig:box_plots_CCR}
	\end{subfigure}
	~
	\begin{subfigure}[t]{0.25\linewidth}\centering
		\centering
		\includegraphics[width=\textwidth,height=1.15in]{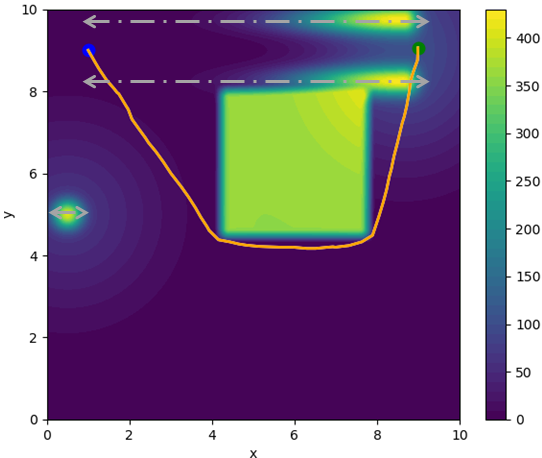}
		\caption[Caption]{Paths returned by CPT-RRT*(Orange) in risk averse perception of CCR map with $\Theta=\{0.74,1,0.88,6.25\}$}
		\label{fig:CCR_CPT_map}
	\end{subfigure}%
	\caption[Caption]{Comparison with Risk-RRT*. The CCR Map containing 3 moving humans and a stationary obstacles at initial time is shown in the background.  }
	\label{fig:result_dynamic_comparison}
\end{figure*}

%\paragraph*{Variation in $\delta$} Using the previous environment
%(Fig.~\ref{fig:fire_env}) and a cost and uncertainty averse profile
%(Fig.~\ref{fig:fire_env_cpt_risk_averse})), we run CPT-RRT* with
%$\delta$ varying from $\delta=10^{-6}$ to $\delta=10^{2}$ for
%$T=15,000$ iterations. The results are shown in
%Fig.~\ref{fig:delta_paths}. We can see that when $\delta \geq 1$ the
%path output changes reflecting an increase in urgency over risk and
%thus choosing shorter paths. When $\delta$ is comparable to the risk
%values (in this case $\delta=100$), we see that the paths no longer
%avoid the high risk area and can even go through the soft
%obstacles. From this study, we observe that $\delta$ needs to be
%rather small as compared to the given risk profile in order to ensure
%meaningful consideration of risk in the planning process. If explicit
%obstacle avoidance is a necessity, then a standard collision check can
%be performed prior to adding a node in the tree $\mathcal{G}$.
%\begin{figure}
%	\centering
%	\includegraphics[width=0.79\linewidth]{delta_results.png}
%	\caption{Paths produced by varying $\delta$}
%	\label{fig:delta_paths}
%\end{figure}
\paragraph*{Variation in $\delta$}  \color{black} Using the previous environment
(Figure~\ref{fig:fire_env}) and a cost and uncertainty averse profile
(Figure~\ref{fig:fire_env_cpt_risk_averse})), we run CPT-RRT* with
$\delta$ varying from $\delta=10^{-6}$ to $\delta=10^{2}$ for
$T=15,000$ iterations. The results are shown in
Figure~\ref{fig:delta_paths}. We can see that when $\delta \geq 1$ the
path output changes reflecting an increase in urgency over risk and
thus choosing shorter paths. When $\delta$ is comparable to the risk
values (in this case $\delta=100$), we see that the paths no longer
avoid the high risk area and can even go through the soft
obstacles. From this study, we observe that $\delta$ needs to be
rather small as compared to the given risk profile in order to ensure
meaningful consideration of risk in the planning process. If explicit
obstacle avoidance is a necessity, then a standard collision check can
be performed prior to adding a node in the tree $\mathcal{G}$.

%\begin{figure}
%	\centering
%	\includegraphics[width=0.79\linewidth]{delta_results.png}
%	\caption{Paths produced by varying $\delta$}
%	\label{fig:delta_paths}
%\end{figure}

Overall, our adaptation of CPT to the planning setting produces paths
that are logically consistent with a given risk
scenario. Additionally, our planning framework can explore narrow
corridors and cluttered environment and produce smooth paths quickly.

\subsection{CPT planner expressive power evaluation}
%\color{black}
\label{sec:results_cpt_learn}

We now discuss the proposed SPSA framework in
Section~\ref{sec:learning_SPSA} to gauge the adaptability of CPT as a
perception model to depict a drawn path $P_d$.  
To implement SPSA, we follow guidelines from \cite{JCS:03}. We
consider a Bernoulli distribution of $\Delta_k$ with support
$\{-1,1\}$ and equal probabilities, learning rate $a_k =
\frac{0.4}{(1.6 + k )^{0.601}}$ and perturbation parameter $c_k =
\frac{0.97}{(1.6 + k )^{0.301}}$.
%\margina{not sure if I should give all the details as we have
%  shortened section VI.}
We choose $\Theta_0=\{0.74,1,0.88,2.25\}$ for CPT throughout the
simulation, which are the nominal parameters from~\cite{AT-DK:92} and
$q_0=0.5$ for the CVaR variant. We use the same environment as in
Figure~\ref{fig:intro_env_perception} for all the simulations. Four
different paths $\{P^1_d,P^2_d,P^3_d,P^4_d\}$ are drawn by hand on the
expected risk profile (Figure~\ref{fig:fire_env_exp}) using a computer
mouse as shown in Figure~\ref{fig:drawn_paths}. Path $P^1_d$ is
similar to a path generated with expected risk perception
(Figure~\ref{fig:fire_env_exp}). Whereas, path $P^4_d$ and $P^2_d$ are
similar to paths generated with high risk aversion
(Figure~\ref{fig:fire_env_cpt_risk_averse}) and uncertainty
insensitivity (Figure~\ref{fig:fire_env_cpt_unc}) respectively. Path
$P^3_d$ is more challenging to represent as it shows an initial
aversion to risk and uncertainty and then takes a seemingly costlier
turn at the top.
We then use the SPSA approach described in
Section~\ref{sec:learning_SPSA} with a tolerance $\kappa=15$ and a
maximum of $10$ SPSA iterations per trial. We use $T_k=15000$
iterations and $\delta=0.01$ to implement Algorithm~\ref{alg:hi-rrt*}
to determine $P_\Theta$ in order to determine the loss $\Ar$ during
each SPSA iteration. For the CVaR variant, the planner
(Algorithm~\ref{alg:hi-rrt*}) replaces $R^c$ with $R^v$ in order to
use perceived risk according to CVaR while the rest of the RRT*
framework remains unchanged.  At the end each trial we get the area
(loss) $\Ar$ between the returned $P_\Theta$ and the drawn path
$P^x_d$.

\begin{figure*}
\centering
\begin{subfigure}[t]{0.39\linewidth}\centering
\centering
\includegraphics[width=\textwidth]{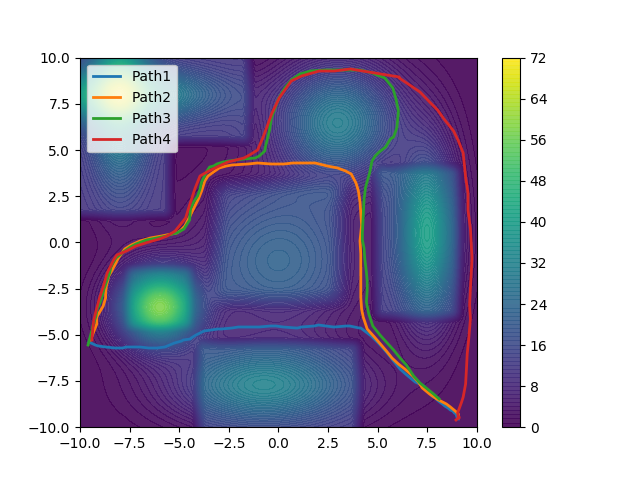}
\caption[Caption]{Four paths $\{P^1_d,P^2_d,P^3_d,P^4_d\}$ are drawn in blue, orange, green and red respectively.}
	\label{fig:drawn_paths}
\end{subfigure}%
~ 
\begin{subfigure}[t]{0.55\linewidth}\centering
\centering
	\includegraphics[width=\textwidth]{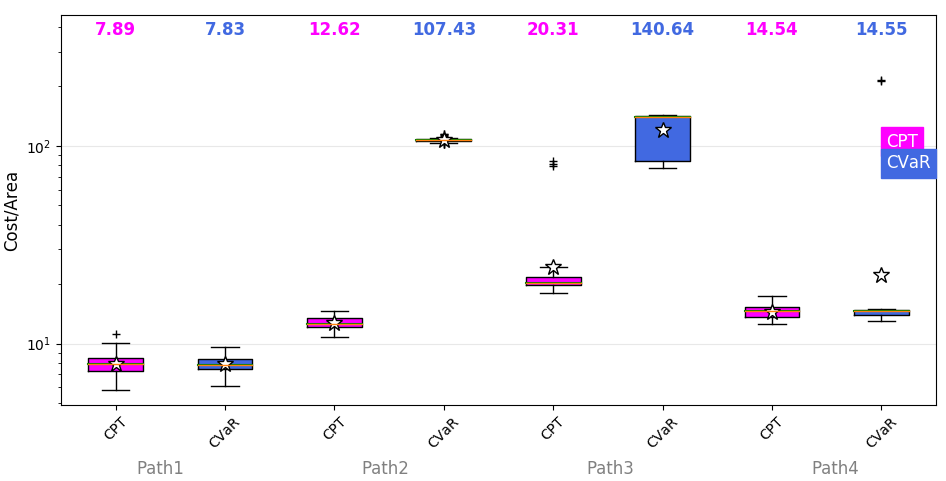}
	\caption{Boxplots showing the cost(Area) returned after using SPSA with CPT and CVaR to capture risk profile of the drawn paths. }
	\label{fig:box_plots}
\end{subfigure}
\caption[Caption]{Result of using CPT and CVaR to model drawn paths.}
\label{fig:result_CPT_eval}
\end{figure*}

We represent the statistics of the returned cost $\Ar$ as boxplots as
shown in Figure~\ref{fig:box_plots}.  Each box plot represents the
distribution of $50$ cost/Area $Ar$ values returned after each trial
for each path and perception model.  The Y-Axis represents the
cost/Area $\Ar$ in a base $10$ log scale.  We calculate a few sample
areas: $\Ar(P^1_d,P^2_d)=99.14, \Ar(P^2_d,P^3_d)=35.20$ and
$\Ar(P^3_d,P^4_d)=73.41$ to give a quantitative idea of the measure
$\Ar$ in this scenario to the reader.  The median values for each box
plot is indicated on the top row. The mean value of the distribution
is indicated as ``stars'', the black lines above and below the box
represent the range, and $+$ indicates outliers.
We observe that from Figure~\ref{fig:box_plots}, both Path $P^1_d$ and
Path $P^4_d$ were captured equally well with CVar and CPT with low
$\Ar$ values. Since both CPT and CVaR are generalizations of expected
risk, paths close (like $P^1_d$) to paths generated from expected risk
can be easily mimicked. Similarly, since CPT and CVaR are designed to
capture risk aversion, paths close (like $P^4_d$) to risk averse paths
(Figure~\ref{fig:fire_env_cpt_risk_averse}) can also be easily
captured.

However, we see a contrast in performance for path $P^2_d$ and path
$P^3_d$. CPT, on both occasions, is able to track the drawn paths
reasonably well with low $\Ar$ values. Whereas CVaR has consistently
higher (an order of magnitude) $\Ar$ values, indicating the inability
to capture the risk perception leading to path $P^2_d$ and path
$P^3_d$. This is due to the fact that CPT can handle uncertainty
perception independently from the cost (as seen between
Figure~\ref{fig:intro_CPT_unc1} and
Figure~\ref{fig:intro_CPT_unc2}). This ability is needed to capture
paths like $P^2_d$ and $P^3_d$ which is lacking in models like CVaR
and expected risk. This shows the generalizability of CPT over CVaR
with CPT having a richer modeling capability, thus validating the theoretical results from Lemma~\ref{lem:CPT_expressiveness}.

%We generate trajectories in this
%environment using Algorithm~\ref{alg:CPT_env}~and~\ref{alg:hi-rrt*}
%for different initial conditions $x_s$,  $x_g$, and CPT parameters
%$\Theta_d$ which serve as ground truth. We use $P_d$ which denotes the
%path generated by parameters $\Theta_d$ as input to the proposed SPSA
%learner in Section~\ref{sec:learning_SPSA}. We then evaluate the
%estimated parameter $\Theta$ from applying
%Algorithm~\ref{alg:CPT_SPSA} and compute $\|P_d - P_{\Theta_{k+1}} \|$
%which serves as the noisy measurement of the loss function. 
%\paragraph{Simulations}
%\label{sec:results_cpt_learn_sims}

\color{black}

\section{Conclusions and Future Work }
\label{sec:conclussions}

We have proposed a novel adaptation of CPT to model a
DM's non-rational perception of a risky environment in the context of
path planning.
%The approach combines
%diverse tools from Path Planning, Behavioral Economics, Optimization
%and Robotics for this purpose. 
Firstly, using CPT, we
provide a tuning knob to model various risk
perceptions of an uncertain spatial cost. Next, we demonstrate a novel embedding of
non-rational risk perception into a sampling based planner, the
CPT-RRT*, to plan
asymptotically optimal paths in perceived environments. Finally, we theoretically and empirically evaluate CPT as a good
approximator to the risk perception of arbitrary drawn paths by
comparing against CVaR, and show that CPT is a richer model
approximator. Future work will analyze how CPT can be used to learn
the risk profile of a DM using learning frameworks and conducting user studies.

\bibliographystyle{IEEEtran}
 \bibliography{alias,SM,SMD-add}
\end{document}